\documentclass[10pt]{article} % For LaTeX2e
\usepackage[preprint]{tmlr}

\usepackage{amsmath,amsfonts,bm}

\def\eqref#1{equation~\ref{#1}}
\def\1{\bm{1}}

\DeclareMathAlphabet{\mathsfit}{\encodingdefault}{\sfdefault}{m}{sl}
\SetMathAlphabet{\mathsfit}{bold}{\encodingdefault}{\sfdefault}{bx}{n}

\usepackage{hyperref}
\usepackage{url}

\usepackage{booktabs}						% professional-quality tables
\usepackage{multirow}						% tabular cells spanning multiple rows
\usepackage{amsfonts}						% blackboard math symbols
\usepackage{graphicx}						% figures
\usepackage{duckuments}						% sample images
\usepackage{subcaption}

\usepackage{wrapfig}
\usepackage{amsthm}
\newtheorem{theorem}{Theorem}
\newtheorem{lemma}{Lemma}
\usepackage{enumitem}

\title{PLGC: Pseudo-Labeled Graph Condensation}

\author{\name Jay Nandy \email jayjaynandy@gmail.com \\
      \addr Ex Fujitsu Research of India
      \AND
      \name Arnab Kumar Mondal \email arnabkumarmondal123@gmail.com\\
      \addr Ex Fujitsu Research of India 
      \AND
      \name Anuj Rathore \email anuj.rathore@fujitsu.com\\
      \addr Fujitsu Research of India 
      \AND
      \name Mahesh Chandran \email mahesh.chandran@fujitsu.com\\
      \addr Fujitsu Research of India 
      }

\def\month{MM}  % Insert correct month for camera-ready version
\def\year{YYYY} % Insert correct year for camera-ready version
\def\openreview{\url{https://openreview.net/forum?id=XXXX}} % Insert correct link to OpenReview for camera-ready version

\begin{document}

\maketitle

\begin{abstract}
Large graph datasets make training graph neural networks (GNNs) computationally costly. Graph condensation methods address this by generating small synthetic graphs that approximate the original data. However, existing approaches rely on clean, supervised labels, which limits their reliability when labels are scarce, noisy, or inconsistent.
We propose Pseudo-Labeled Graph Condensation (PLGC), a self-supervised framework that constructs latent pseudo-labels from node embeddings and optimizes condensed graphs to match the original graph’s structural and feature statistics—without requiring ground-truth labels.
PLGC offers three key contributions: (1) A diagnosis of why supervised condensation fails under label noise and distribution shift. (2) A label-free condensation method that jointly learns latent prototypes and node assignments. (3) Theoretical guarantees showing that pseudo-labels preserve latent structural statistics of the original graph and ensure accurate embedding alignment.
Empirically, across node classification and link prediction tasks, PLGC achieves competitive performance with state-of-the-art supervised condensation methods on clean datasets and exhibits substantial robustness under label noise, often outperforming all baselines by a significant margin. 
Our findings highlight the practical and theoretical advantages of self-supervised graph condensation in noisy or weakly-labeled environments.
%\footnote{Code Link: \href{https://anonymous.4open.science/r/PLGC-0B26/}{https://anonymous.4open.science/r/PLGC-0B26/}}.
\end{abstract}

\maketitle
\section{Introduction}  \label{sec_intro}

Large-scale graph-structured data arise in numerous domains, including social networks \cite{social_www_2019}, biological systems \cite{bio_bio_2021}, and knowledge graphs \cite{kd_neurips_2020}. The growing scale and complexity of these graphs, often comprising millions of nodes with high-dimensional attributes and heterogeneous connectivity patterns, pose significant computational challenges for training graph neural networks (GNNs). These challenges are exacerbated in resource-constrained environments and in settings requiring repeated training, such as hyperparameter search, architecture exploration, or continual adaptation to dynamic graph streams.

Graph condensation has emerged as an effective strategy to mitigate these limitations by synthesizing a compact graph whose structural and feature distributions approximate those of the original data \cite{gcond,sfgc,geom,sgdd,sun2024gc}. Models trained on the condensed graph attain accuracy close to those trained on the full graph while reducing training and storage costs substantially. Despite these advantages, existing graph condensation methods predominantly depend on high-quality supervised labels to guide the construction of the synthetic graph. This reliance significantly restricts their applicability. In many real-world scenarios---financial transaction networks \cite{sheng2011impact}, dynamic social interactions \cite{wang2015localizing}, and sensor or IoT graphs \cite{yick2008wireless,dai2021nrgnn,ding2024divide}---labels are scarce, incomplete, noisy, or subject to temporal distribution shift. Such conditions undermine the stability and effectiveness of label-dependent condensation objectives.

\begin{figure}%{r}{0.65\textwidth}
  \centering
  \includegraphics[keepaspectratio, width=0.8\textwidth]{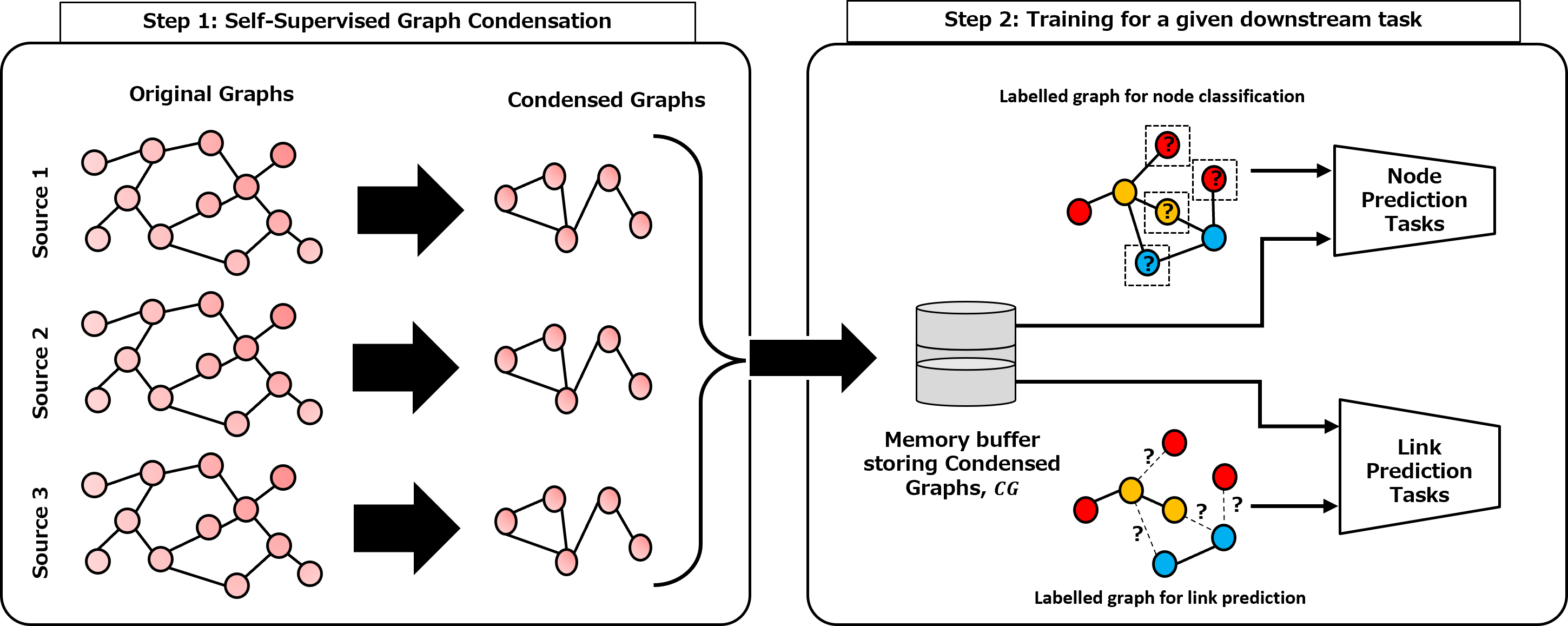}
  \caption{\small Proposed self-supervised graph condensation framework to condense the original (unlabeled) graphs from different distributed sources within a small memory that can be used for training or finetuning for downstream tasks with limited supervision.}
  \label{fig_ssgc}
\end{figure}

To address these limitations, we propose \textbf{Pseudo-Labeled Graph Condensation (PLGC)}, a self-supervised condensation framework that circumvents the need for ground-truth labels by constructing and optimizing \emph{latent pseudo-labels} that capture the intrinsic structural and feature characteristics of the original graph. PLGC operates through an alternating optimization procedure consisting of two components:  
\begin{enumerate}[leftmargin=*]
\item \textbf{Pseudo-Label Construction}, in which latent soft labels are inferred by encoding statistical relationships in node features and neighborhood structures; and  
\item \textbf{Condensed Graph Optimization}, in which the synthetic graph is updated by minimizing the divergence between its embedding distribution and the pseudo-label-induced representation of the original graph.
\end{enumerate}
This alternating mechanism ensures that the synthetic graph remains aligned with the latent geometry of the original graph, regardless of supervised information. Furthermore, PLGC extends naturally to \textbf{multi-source condensation}, where multiple independent graphs are available. In such cases, PLGC learns source-specific condensed graphs while jointly aligning their latent representations through pseudo-labels, thereby enabling the training of a unified downstream model in a fully or partially unsupervised manner (Figure~\ref{fig_ssgc}).

A key contribution of this work is a \textbf{theoretical characterization of pseudo-label stability and condensation error propagation}. We establish (i) conditions under which pseudo-label estimation remains stable in the presence of noise or distributional variation, and (ii) upper bounds on the embedding divergence between the original and condensed graphs induced by PLGC’s optimization dynamics. These results formally justify the robustness of PLGC in scenarios where supervised condensation methods deteriorate.

Our primary contributions are as follows:
\begin{enumerate}[leftmargin=*]
    \item We identify structural limitations of supervised graph condensation, demonstrating analytically and empirically that their effectiveness deteriorates under label noise, label scarcity, and label distribution shift.
    
    \item We introduce \textbf{PLGC}, a self-supervised condensation framework based on alternating pseudo-label estimation and condensed-graph optimization, completely eliminating the need for ground-truth labels.
    
    \item We provide \textbf{theoretical guarantees}, including stability results for pseudo-label inference and error bounds for condensed--original embedding alignment, establishing principled robustness under weak or unreliable supervision.
    
    \item We conduct extensive experiments on five benchmark datasets, showing that  
    (i) PLGC significantly outperforms supervised condensation methods in low-label, noisy-label, and shift-label regimes;  
    (ii) it remains competitive with supervised methods when clean labels are available; and  
    (iii) it offers superior generalization in multi-source and distributionally heterogeneous settings.
\end{enumerate}

\section{Background \& Related Work}\label{sec_related_work}

\subsection{Problem Formulation}
\label{sec_problem_statement}

We consider a setting with $M$ independent sources, each providing an unlabeled graph 
$\{\mathcal{T}^{i}\}_{i=1}^{M}$ defined over a shared \emph{node-level} label space.  
Each graph $\mathcal{T}^{i} = (X^{i}, A^{i})$ contains $N^{i}$ nodes with node features 
$X^{i} \in \mathbb{R}^{N^{i} \times d}$ and adjacency matrix $A^{i} \in \{0,1\}^{N^{i} \times N^{i}}$, but no ground-truth labels.  
The objective is to construct, for each source graph, a significantly smaller \emph{condensed graph} 
paired with a set of \emph{pseudo-labels},
$\{(\mathcal{S}^{i}, \widetilde{Y}^{i})\}_{i=1}^{M}$,
such that the condensed graphs retain the essential structural and feature-level information required for downstream predictive tasks.

Formally, each condensed graph $\mathcal{S}^{i} = (\widetilde{X}^{i}, \widetilde{A}^{i})$ contains 
$N'^{\,i}$ synthetic nodes, where $N'^{\,i} \ll N^{i}$, and is associated with pseudo-labels 
$\widetilde{Y}^{i} = \{\widetilde{y}^{i}_{k}\}_{k=1}^{N'^{\,i}}$ that summarize the latent cluster-level 
statistics of the original graph.
Our goal is to learn a representation model $\mathrm{GNN}_{\theta}$ such that:

\begin{enumerate}[leftmargin=*]
    \item \textbf{Condensed graph alignment.}  
    The condensed graphs can be used to train $\mathrm{GNN}_{\theta}$ in a self-supervised manner by 
    aligning node embeddings with the pseudo-labels:
    \[
    \mathrm{GNN}_{\theta}(\mathcal{S}^{i}) \rightarrow \widetilde{Y}^{i}, 
    \quad \forall\, i \in \{1,\dots,M\}.
    \]

    \item \textbf{Downstream transferability.}  
    Once $\mathrm{GNN}_{\theta}$ is learned from synthetic graphs, it can be transferred to a new 
    downstream task by fine-tuning only a lightweight prediction head $f_{\phi}$ using limited supervision:
    \[
    f_{\phi}\big(\mathrm{GNN}_{\theta}(\mathcal{G}^{\text{test}})\big) \rightarrow 
    Y^{\text{test}},
    \]
    where $\mathcal{G}^{\text{test}}$ is a test graph with available labels.
\end{enumerate}

Thus, the overall goal is to design a self-supervised graph condensation framework that produces 
compact synthetic graphs and pseudo-labels that preserve task-relevant latent structure, enabling 
efficient and minimally supervised training of downstream graph models.

\begin{table*}
\centering
\resizebox{0.9\textwidth}{!}{
\begin{tabular}{l|lccccc}
\hline
& \multicolumn{1}{c}{\textbf{Objective}} 
& \textbf{\begin{tabular}[c]{@{}c@{}}Condensed\\ Memory\end{tabular}}
& \textbf{\begin{tabular}[c]{@{}c@{}}Self\\supervision\end{tabular}} 
& \textbf{\begin{tabular}[c]{@{}c@{}}Multiple \\ Sources\end{tabular}}
& \textbf{\begin{tabular}[c]{@{}c@{}}Downstream \\ Performance\end{tabular}} 
& \textbf{\begin{tabular}[c]{@{}c@{}}Robustness under \\ label-noise\end{tabular}}
\\ \hline
\textbf{\begin{tabular}[c]{@{}l@{}}Supervised \\ Condensation\\ (Existing)\end{tabular}} & \begin{tabular}[c]{@{}l@{}}Produce a synthetic graph \\ by preserving class-specific \\ characteristics.\end{tabular} 
& Yes & No & Yes & Good & Poor\\ \hline
\textbf{\begin{tabular}[c]{@{}l@{}}Coarsening \& \\ Sparsification\end{tabular}} & \begin{tabular}[c]{@{}l@{}}Returns a small graph by \\preserving certain graph \\properties\end{tabular} 
& Yes & Yes & Yes & Poor & Poor\\ \hline
\textbf{Graph SSL} & \begin{tabular}[c]{@{}l@{}}Learns a generalized \\ representation model for \\ various downstream tasks.\end{tabular} 
& No & Yes & No & Good & Good \\ \hline
\textbf{\begin{tabular}[c]{@{}l@{}}PLGC (Ours)\end{tabular}} & \begin{tabular}[c]{@{}l@{}}Returns a small synthetic \\ graph by preserving latent \\ statistics.\end{tabular} 
& Yes & Yes & Yes & Good & Good
\\ \hline
\end{tabular}
}
\caption{\small Advantages of our pseudo-labeled-based self-supervised graph condensation over existing techniques.}
\label{table_advantagePLGC}
\vspace{-1em}
\end{table*}

\textit{\textbf{Graph reduction techniques}}
aim to reduce the size of a given graph while preserving essential information, and can be broadly categorized into graph sparsification, coarsening, sketching, and graph condensation. 
Graph sparsification selects a subset of nodes or edges from the original graph to reduce complexity \citep{chen2021unified,razin2023ability}. 
Graph coarsening constructs a smaller graph by learning a surjective mapping from the original graph, typically by merging multiple nodes into supernodes \citep{si2022serving,kumar2023featured}. 
Graph sketching summarizes the original graph into a compact representation by preserving selected structural properties (\textit{e.g.}, centrality or spectral statistics) through node or edge sampling \citep{ding2022sketch}.

Although these approaches are often unsupervised or require minimal supervision, they face notable limitations in practical learning scenarios. For instance, graph sparsification becomes less effective when nodes carry rich attributes, and coarsening or sketching methods typically preserve task-agnostic properties (\textit{e.g.}, eigenvalues or degree distributions) that may not align with downstream objectives. As a result, their performance often lags behind graph condensation methods when used for training graph neural networks.

\textbf{\textit{Graph condensation (supervised)}} methods aim to distill large graphs into small synthetic datasets that preserve latent structural and semantic information, such that models trained on condensed graphs and perform comparably to those trained on the full data \citep{gcond,doscond,gcdm,sfgc,geom,cat,puma,lei2023comprehensive,yu2023dataset}. 
Early work focused on gradient matching between real and synthetic data \citep{wang2018dataset,nguyen2020dataset,zhao2020dataset}, initially in continuous domains (e.g., images) and later extended to graphs.

GCond \citep{gcond} adapts online gradient matching to graph-structured data by jointly learning node features and adjacency matrices. DosCond \citep{doscond} improves efficiency via single-step gradient matching using probabilistic graph models for graph classification. Subsequent approaches further reduce computation by implicitly encoding topology into node features, often using identity or fixed adjacency structures.
Representation matching methods, such as CaT \citep{cat}, PUMA \citep{puma}, SERGCL \citep{sergcl} generate condensed graphs by aligning class-wise representations between real and synthetic graphs. Trajectory-based approaches, including SFGC \citep{sfgc} and GEOM \cite{geom}, leverage expert training trajectories from original graphs to guide condensation. While effective, these trajectory-based methods incur substantial computational and memory overhead due to the storage and replay of their expert trajectories.
ST-GCOND \citep{stgcond} improves efficiency through spectral approximations but remains fundamentally supervised, limiting its robustness in noisy or low-label settings.

A fundamental limitation of most of these existing graph condensation methods is their reliance on labeled data. In practice, node or graph labels may be noisy, incomplete, or unavailable, which significantly degrades performance. Even automatically generated labels can be corrupted by measurement errors (e.g., faulty sensors in weather data \citep{weather_sensor}) or data /distribution shifts in evolving networks \citep{chauhan2022multi,graph_distShift}.

\textit{\textbf{Graph Condensation \& Self-Supervised Learning (SSL).}}
Graph self-supervised learning (GSSL) enables representation learning without labeled supervision and has demonstrated strong generalization across downstream tasks \cite{kim2021how,GraphCL,tgcl}. SSL methods generally fall into two categories: 
(1) predictive approaches \cite{Hu2020Strategies,kim2021how,rong2020self}, which construct pretext tasks based on structural signals and emphasize local semantics, and 
(2) contrastive learning approaches \cite{GraphCL,JOAO,yin2022autogcl}, which maximize agreement between augmented views to capture global semantics. 
However, SSL methods typically output representation models rather than condensed datasets, and they do not directly address multi-source knowledge aggregation or data compression.

Recent work has begun exploring self-supervised formulations for graph condensation, although important gaps remain. SGDC \citep{wang2024self} performs condensation for graph-level datasets by matching representations from a fixed self-supervised encoder, but it does not infer latent cluster structure, lacks node-level pseudo-labeling, and does not support multi-source alignment. CTGC \citep{CTGC} incorporates contrastive objectives to disentangle semantic and structural signals, yet it does not construct explicit prototype-based pseudo-labels nor provide theoretical guarantees on cluster preservation or assignment stability. 

In contrast, our PLGC introduces a principled pseudo-labeling mechanism that jointly learns latent prototypes and node–prototype assignment matrices in a fully self-supervised manner. This design enables stable, label-free representation matching with theoretical guarantees on cluster separation and centroid concentration, and yields strong robustness under label noise. Moreover, PLGC naturally extends to multi-source scenarios by aligning source-specific condensed graphs within a shared latent space—capabilities not addressed by prior methods. Overall, PLGC provides a unified, theoretically grounded, and practically robust framework that overcomes key limitations of existing supervised and self-supervised condensation approaches.

Table~\ref{table_advantagePLGC} summarizes the strengths and limitations of existing methods. While graph condensation has broad applications—including architecture search, privacy preservation, adversarial robustness, and continual learning—prior approaches either rely heavily on labels or lack reusable condensed representations. PLGC bridges this gap by combining the generalization strength of SSL with the efficiency of condensation, enabling reusable condensed graphs and effective knowledge aggregation from noisy or unlabeled multi-source data.

\section{Proposed Method}\label{sec:proposed_method}
Existing graph condensation methods are typically supervised, that rely on ground-truth labels to guide either gradient matching or representation matching during the construction of a synthetic graph 
\cite{gcond,cat}.  
Their objective can be written as
\begin{align}
\label{eq_sup_condensed}
    \min_{\mathcal{S}}~ 
    \mathcal{L}_{node}(\mathrm{GNN}_{\theta_{\mathcal{S}}}(A,X), Y)
    \quad\text{s.t.}\quad
    \theta_{\mathcal{S}} 
    = \arg\min_{\theta}~\mathcal{L}_{node}(\mathrm{GNN}_{\theta}(A', X'), Y'),
\end{align}
where $\mathcal{T}=(A,X,Y)$ denotes the original labeled graph and 
$\mathcal{S}=(A',X',Y')$ is the condensed graph with $N'\ll N$ nodes.  
Here, $A\in\mathbb{R}^{N\times N}$ and $A'\in\mathbb{R}^{N'\times N'}$ denote adjacency matrices;  
$X\in\mathbb{R}^{N\times d}$ and $X'\in\mathbb{R}^{N'\times d}$ denote node features;  
$\mathrm{GNN}_{\theta}$ is a GNN parameterized by $\theta$; and 
$\mathcal{L}_{node}$ is a supervised node-classification loss (e.g., cross-entropy).

The goal in Eq.~\eqref{eq_sup_condensed} is to learn a synthetic graph $\mathcal{S}$ such that 
the model trained on $\mathcal{S}$ generalizes to the original graph $\mathcal{T}$.  
However, the dependence on clean labels $Y$ renders these methods fragile in noisy-label 
scenarios \citep{chauhan2022multi}.  
In contrast, designing a \emph{self-supervised} condensation objective requires eliminating all dependencies on $Y$, still generating a condensed graph that preserves task-relevant structure.

We address this challenge by learning (i) a set of latent \emph{pseudo-labels} $\tilde{Y}$ and 
(ii) an associated \emph{assignment matrix} $Q_{\mathcal{T}}$, enabling self-supervised 
condensation without ground-truth labels.

\subsection{Pseudo-Labeled Graph Condensation}
\label{sec_plgc}

Our method incorporates two key iterative steps for our self-supervised graph condensation:

\begin{enumerate}[leftmargin=*]
    \item \textbf{Updating $K$ pseudo-labels $\tilde{Y}\in\mathbb{R}^{K\times d}$},  
    where $K\ll N$, such that each pseudo-label captures representative statistics of one latent 
    cluster in the embedding space.

    \item \textbf{Updating assignment matrix 
    $Q_{\mathcal{T}}\in\{0,1\}^{N\times K}$},  
    where each row is a one-hot vector linking each node to exactly one pseudo-label.
\end{enumerate}

The pseudo-label matrix $\tilde{Y}=[\tilde{y}_1,\dots,\tilde{y}_K]^\top$ represents $K$ distinct pseudo-labels, each capturing a cluster-level embedding in the original graph.  
Given $Q_{\mathcal{T}}$, the induced pseudo-label for each node is $Q_{\mathcal{T}}\tilde{Y}$.

\begin{figure}%{r}{0.65\textwidth}
  %\vspace{-1.5em}
  \centering
  \includegraphics[width=0.85\textwidth, keepaspectratio]{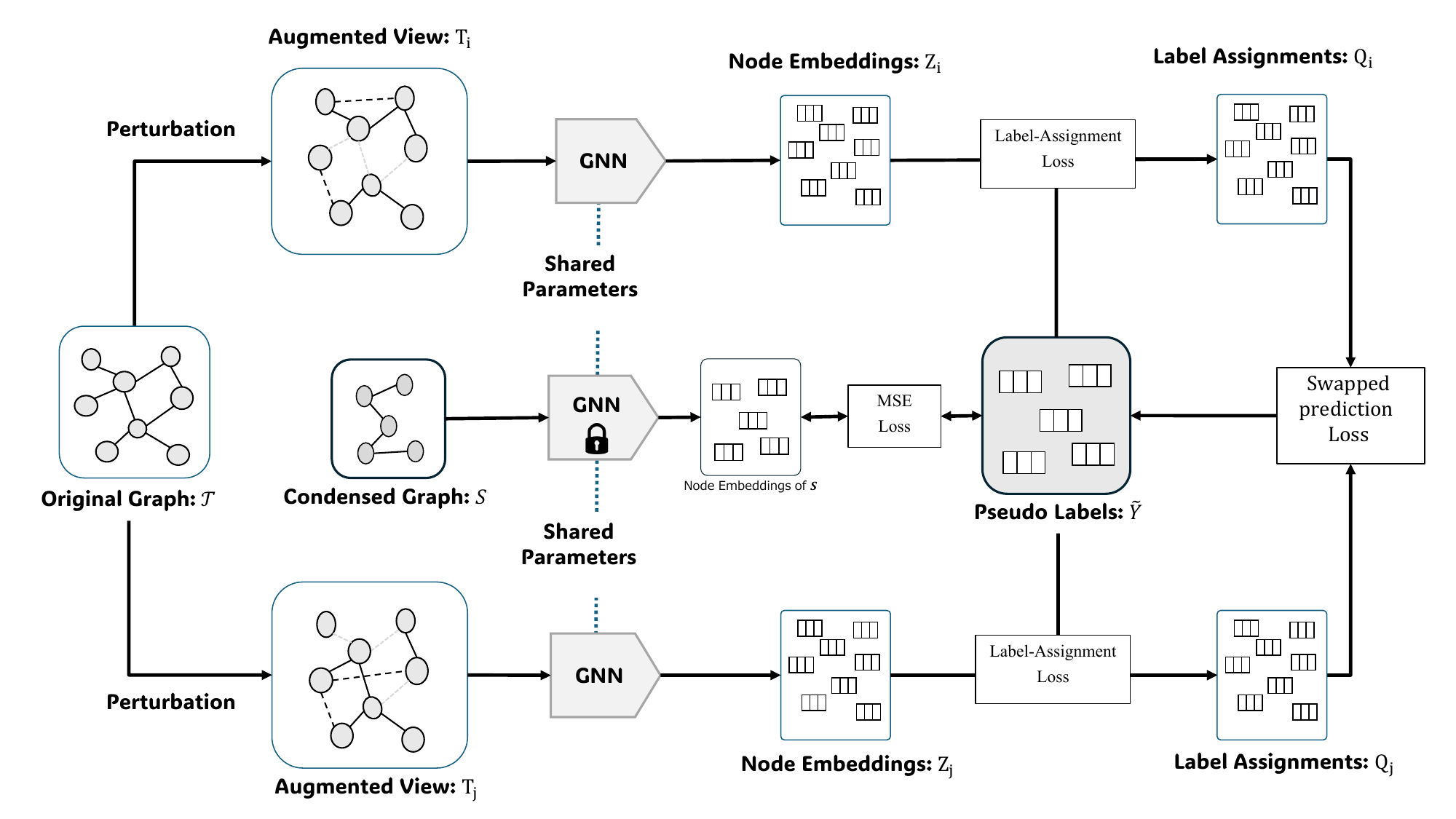} 
\caption{\small 
Overview of the proposed PLGC framework. The method alternates between two coupled optimization steps: 
\textbf{(I) Pseudo-Label Learning} — multiple graph augmentations are generated from the original graph, processed through a shared GNN encoder to obtain node embeddings, and assigned to the pseudo-labels using an entropy-regularized Sinkhorn optimization. The resulting assignments and pseudo-labels are updated via a swapped-assignment view prediction loss. 
\textbf{(II) Condensed Graph Optimization} — the condensed graph is passed through the same shared encoder, and its node embeddings are aligned with the learned pseudo-labels using an MSE-based representation matching loss. 
Together, these steps iteratively refine pseudo-labels and the condensed graph so that the synthetic graph preserves the latent structure of the original graph without relying on ground-truth labels.}
  \label{fig_plgc}
  % \vspace{-1.5em}
\end{figure}

\paragraph{Self-supervised condensation objective.}  
We formulate the PLGC objective as an unsupervised analogue of Eq.~\eqref{eq_sup_condensed},  
adding two auxiliary optimization problems to learn $Q_{\mathcal{T}}$ and $\tilde{Y}$:
\begin{align}
\label{eq_unsup_condensed}
\min_{\mathcal{S}}~
    \mathcal{L}\big(\mathrm{GNN}_{\theta_{\mathcal{S}}}(A,X), Q_{\mathcal{T}}\tilde{Y}\big)
    \quad\text{s.t.}\quad
    \theta_{\mathcal{S}} 
    = \arg\min_{\theta}~
        \mathcal{L}\big(\mathrm{GNN}_{\theta}(A',X'), Q_{\mathcal{S}}\tilde{Y}\big),
\end{align}
where $\mathcal{L}$ measures the discrepancy between node embeddings and pseudo-labels.
\begin{align}
\label{eq_unsup_second}
    Q_{\mathcal{T}},\tilde{Y}
    = \arg\min_{Q_{\mathcal{T}},\tilde{Y}}
        \mathcal{L}_{pseudo}\big( \mathrm{GNN}_{\theta'}(A,X), Q_{\mathcal{T}}\tilde{Y} \big).
\end{align}
where $\mathcal{L}_{pseudo}$ is a self-supervised loss used to learn latent clusters.  

Because every pseudo-label must correspond to at least one node in $\mathcal{S}$, we set $N'=K$ and fix the condensed-graph assignment matrix as
$Q_{\mathcal{S}} = I_{K}$, ensuring a one-to-one mapping between nodes of $\mathcal{S}$ and pseudo-labels.
Figure~\ref{fig_plgc} illustrates the PLGC pipeline.

\subsubsection{Learning Pseudo-Labels and Assignment Matrix}
\label{sec_pseudo_label}

We construct $\mathcal{L}_{pseudo}$ using two iterative components:
(a) a \emph{swapped-assignment view prediction loss} $\ell_{swap}$, used to update $\tilde{Y}$ and the encoder $\mathrm{GNN}_{\theta'}$, and
(b) An \emph{assignment loss} $\ell_{assign}$, used to compute $Q_{\mathcal{T}}$ by enforcing balanced assignment of nodes to the pseudo-labels.
We apply two stochastic augmentations $T_i$ and $T_j$ to the original graph and extract 
$\ell_2$-normalized embeddings:
$\mathcal{Z}_i=\frac{\mathrm{GNN}_{\theta'}(T_i)}{\|\mathrm{GNN}_{\theta'}(T_i)\|_2}$ and 
$\mathcal{Z}_j=\frac{\mathrm{GNN}_{\theta'}(T_j)}{\|\mathrm{GNN}_{\theta'}(T_j)\|_2}$.

\paragraph{Swapped-assignment view prediction loss}
contrasts between the embeddings of different views by comparing their pseudo-label assignments instead of their features, allowing to update the pseudo-labels in an online fashion \citep{swav}.
Given node embedding $z_{i,n}$ from augmentation $T_i$ and pseudo-label 
$q_{j,n}$ computed from $T_j$, the loss encourages consistent cluster assignments across 
augmentations:
\begin{align}
    \ell_{swap}(z_{i,n}, q_{j,n})
    = -\sum_{k} q_{j,n}^{(k)}
        \log
        \frac{
            \exp(z_{i,n}^\top \tilde{y}_k)
        }{
            \sum_{k'} \exp(z_{i,n}^\top \tilde{y}_{k'})
        }.
\end{align}

\paragraph{Balanced assignment loss} ensures that the node-embeddings are equally distributed among the pseudo-labels in a batch are distinct, therefore preventing the trivial solution of producing same assignments for each node.
Following \cite{asano2019self,swav}, we compute $Q_i$ for augmentation $T_i$ by solving
\begin{align}
\label{eq_assignment}
\begin{aligned}
    \ell_{assign}
    &= \mathrm{Tr}(Q_i^\top \tilde{Y}^\top \mathcal{Z}_i)
    + \epsilon\,\mathcal{H}(Q_i), 
    \quad\text{s.t.}~~ 
    &Q_i \in \Big\{
        Q_i\ge0:
        Q_i^\top \mathbf{1}_B = \tfrac{1}{K}\mathbf{1},
        \quad
        Q_i\mathbf{1}_K = \tfrac{1}{B}\mathbf{1}
        \Big\},
\end{aligned}
\end{align}
where $B$ is the batch size and $\mathcal{H}$ is entropy.  
The solution is given by
$Q_i^* 
    = \mathrm{diag}(u)
        \exp\!\big(\tilde{Y}\mathcal{Z}_i/\epsilon\big)
        \mathrm{diag}(v)$,
with $u,v$ obtained via Sinkhorn-Knopp normalization \cite{cuturi2013sinkhorn}.  
Finally, each row is discretized to the closest one-hot vector.
The full pseudo-label learning loss is as follows:
\begin{align}
\label{eq_pseudo_label_learner}
\begin{aligned}
    \mathcal{L}_{pseudo}
    = \sum_{i,j}\sum_{n}
        \ell_{swap}(z_{i,n}, q_{j,n})
        \quad
        \text{s.t.}\quad
        Q_i = \arg\max_Q \ell_{assign}.
\end{aligned}
\end{align}

\subsubsection{Generating Condensed Graphs}
\label{sec_ssl_graphgen}

After learning $(Q_{\mathcal{T}},\tilde{Y})$ and encoder $\mathrm{GNN}_{\theta'}$ using 
$\mathcal{L}_{pseudo}$, we construct $\mathcal{S}$ by minimizing a representation-matching objective.

While prior condensation methods match gradients \cite{gcond} or class-conditional distributions 
\cite{cat}, we adopt latent representation matching via 
\emph{maximum mean discrepancy (MMD)} \cite{zhao2023DM}.  
The supervised version aligns conditional means 
$\mathbb{E}[Z_{\mathcal{T}\mid y}]$ and 
$\mathbb{E}[Z_{\mathcal{S}\mid y}]$ for each class $y$.

For PLGC, we replace class labels with pseudo-labels and use 
$\mathrm{GNN}_{\theta'}$ as the embedding model.  
Because $Q_{\mathcal{S}}=I_K$ and $N'=K$, each pseudo-label corresponds to a unique synthetic node.  
Thus, the MMD objective reduces to a mean-square loss:
\begin{align}
\min_{\mathcal{S}}
    \sum_{\tilde{y}\in\tilde{Y}}
        \big\|
            \tilde{y} - z_{\mathcal{S}\mid \tilde{y}}
        \big\|^2,
\end{align}
where $z_{\mathcal{S}\mid\tilde{y}}$ denotes the embedding of the synthetic node mapped to $\tilde{y}$.  
Consistent with findings in \cite{cat,sfgc,geom}, we omit adjacency learning and condense only the node features, benefiting both performance and storage efficiency.

\subsection{Backbone Reconstruction and Fine-Tuning}

Given $M$ sources, PLGC produces $M$ condensed graphs with associated pseudo-labels: $CG = \{(\mathcal{S}_i, \tilde{Y}_i)\}_{i=1}^M$.

\paragraph{(a) Backbone reconstruction.}
We train a backbone GNN by aligning embeddings of condensed graphs with their pseudo-labels:
    $\min_{\theta}
    \sum_{(\mathcal{S}_i,\tilde{Y}_i)\in CG}
        \sum_{\tilde{y}\in\tilde{Y}_i}
            \big\|
                \tilde{y} - z_{\mathcal{S}_i\mid\tilde{y}}
            \big\|^2$.

\paragraph{(b) Task-specific fine-tuning.}
For downstream applications, we fine-tune a prediction head $f_\phi$ on 
task-specific graph $\mathcal{G}^{test}$ using clean labels $Y^{test}$:
$\min_{\phi}
    \mathcal{L}_{downstream}\big(
        f_{\phi}\circ \mathrm{GNN}_{\theta}(A^{test},X^{test}), 
        Y^{test}
    \big)$.

Therefore, PLGC requires no labels during condensation and uses only minimal supervision during downstream fine-tuning.

% ====================================================

\section{Theoretical Motivation and Statement}\label{sec:theory}

For graph condensation to remain effective across downstream tasks and varying noise levels, the synthetic graph must faithfully preserve the statistical geometry of the original graph. Since PLGC operates without access to any domain specific ground-truth class labels, the quality of the learned pseudo-labels is critical -- we should ensure the pseudo-labels remain aligned with the underlying latent structure to obtain informative, stable, and robust condensed graphs.
Our theoretical analysis formalizes this intuition under with the following mild and widely adopted assumptions in centroid-based clustering theory:
\begin{itemize}[leftmargin=*]
    \item[A1] \textbf{Sub-Gaussian Latent Structure:} A spherical sub-Gaussian latent structure, where the node embeddings within each latent cluster exhibit Gaussian-type tail decay (concentration) without requiring a specific distributional form.
    This assumption holds for Gaussian mixtures, bounded-support distributions, and more generally for any sub-Gaussian cluster distribution, which is standard in high-dimensional clustering analysis.
    
    \item[A2] \textbf{Separability:} Underlying true latent centers, $\{\mu_k\}$ are well-separated by a non-trivial minimum distance: $\Delta = \min_{j \ne k} \|\mu_j - \mu_k\|$.
\end{itemize}

Under these conditions, we show that the pseudo-labels 
$\widetilde{Y}=\{\widetilde{y}_1,\dots,\widetilde{y}_K\}$ generated by PLGC reliably capture the latent structure of the original graph.  
In essence, we show that the pseudo-labels remain close to the true latent centers and preserve the separation between pseudo-labels.
More formally,

\paragraph{Theoretical statement and assumptions.}
Consider node embeddings  $\{z_i\}_{i=1}^n$ sampled independently corresponding to $K$ latent centers, with pseudo-label assignments denoted by $q_{ik} \in \{0,1\}$.
We set the deviation parameter as: $\epsilon_k \;:=\; 4\sigma\sqrt{\frac{d+\log(2K/\delta)}{s_k}}$ where $s_k := \sum_i q_{ik}$ denotes the effective sample size of cluster $k$.
Assuming (A1) and (A2) and given the pseudo-labels satisfying the stationarity condition (i.e., each pseudo-label aligns with its assignment-weighted feature mean), 
the following statements hold simultaneously for all \(k\) with probability at
least \(1-\delta\):

\begin{enumerate}[label=(\roman*)]
  \item \textbf{Pseudo-label Concentration.} 
  Each pseudo-label satisfies $\|\widetilde{y}_k - \mu_k\| \le \epsilon_k$, i.e., pseudo-labels produced by PLGC remains tightly concentrated to the true population centers.

  \item \textbf{Interior-point recovery.} Any point $z_i$ assigned to $k$-th pseudo-label (i.e., $q_{ik}=1$) such that $\|z_i - \mu_k\| < \frac{\Delta}{2} - \epsilon_k$ remains closer to $\widetilde{y}_k$ than other labels 
  $\widetilde{y}_\ell$ -- indicating that, the interior points are always correctly assigned.

  \item \textbf{Sample Complexity \& Pseudo-Label Separation.}
  A sufficient sample size to ensure a positive interior margin of $\epsilon_{\max}\le\Delta/\beta$ is: $s_k \;\ge\; \frac{16\,\sigma^2\,\beta^2}{\Delta^2}\,\big(d + \log(2K/\delta)\big)$, and leads to a stronger separation of pseudo-labels: $\|\widetilde y_k-\widetilde y_\ell\| \ge (1-\tfrac{2}{\beta})\Delta ~~ \forall k\neq\ell$
\end{enumerate}

The complete proofs and additional supporting lemmas are provided in the appendix.
These results guarantee that pseudo-labels are well separated, accurately concentrated, and robustly recover interior nodes—ensuring that PLGC preserves the underlying latent structure of the full graph.

\section{Experiments}\label{sec_exp}
\begin{wraptable}{r}{0.65\textwidth}
\centering
\resizebox{0.6\textwidth}{!}{
\begin{tabular}{l|ccccc}
\hline
Datasets & \#Nodes & \#Edges & \#Class & \#Feature & Train/Val/Test \\ \hline
Citeseer & 3,327 & 4,732 & 6 & 3,703 & 120/500/1000 \\
Cora & 2,708 & 5,429 & 7 & 1,433 & 140/500/1000 \\
Ogbn-arxiv & 169,343 & 1,166,243 & 40 & 128 & 90,941/29,799/48,603 \\ \hline
Flickr & 89,250 & 899,756 & 7 & 500 & 44,625/22,312/22,313 \\
Reddit & 232,965 & 57,307,946 & 41 & 602 & 15,3932/23,699/55,334 \\ \hline
\end{tabular}
}
\caption{Statistics of the datasets.}\label{table_datasets}
\end{wraptable}

\paragraph{Datasets, Tasks, and Evaluation Protocol.}
We evaluate PLGC on node classification and link prediction tasks under both transductive and inductive settings. 
We use five benchmark datasets: Cora, Citeseer \cite{cora}, and Ogbn-Arxiv \cite{ogbn} for transductive evaluation, and Flickr \cite{flickr} and Reddit \cite{reddit} for inductive evaluation (see Table~\ref{table_datasets}). 
All experiments follow standard data splits and evaluation protocols for each dataset. 
Results are averaged over 10 runs, and we report mean performance with standard deviation.

\paragraph{Label Noise and Multi-Source Settings.}
To study robustness, we introduce node-level label noise during the condensation stage by randomly corrupting a fraction of node labels.
We consider noise levels $\{0.0, 0.3, 0.5, 0.7, 0.9\}$, where $0.0$ denotes fully clean labels.
For multi-source experiments, each dataset is partitioned into three disjoint subgraphs.
In transductive settings, training, validation, and test nodes are distributed across sources, whereas in inductive settings only the training graph is partitioned.

\begin{table*}[t]
\centering
\resizebox{\textwidth}{!}{
% \begin{tabular}{l|c|cccccccc|c|c}
\begin{tabular}{l|c|ccccc|cccc|c}
\hline

 & r & \multicolumn{5}{c|}{\textbf{Supervised}} & \multicolumn{4}{c|}{\textbf{Self-Supervised}} & \begin{tabular}[c]{@{}l@{}}Whole Dataset\end{tabular} \\ \cline{3-11}
 &   & DCGraph & GCond & GCond-X & SFGC & GEOM & Random & Herding & K-Center & \textbf{PLGC (Ours)} & \\ \hline

%  & r & Random & Herding & K-Center & DCGraph & GCond & GCond-X & SFGC & GEOM & \begin{tabular}[c]{@{}l@{}}PLGC  (Ours)\end{tabular} & \begin{tabular}[c]{@{}l@{}}Whole Dataset\end{tabular} \\ \hline
\multirow{3}{*}{ Citeseer}
& 0.9\% & 66.8$_{\pm 1.5}$ & 70.5$_{\pm 1.2}$ & 71.4$_{\pm 0.8}$ & 71.4$_{\pm 0.5}$ & 73.0$_{\pm 0.5}$ & 54.4$_{\pm 4.4}$ & 57.1$_{\pm 1.5}$ & 52.4$_{\pm 2.8}$  & \textbf{68.3$_{\pm 1.2}$} & \multirow{3}{*}{ 71.7$_{\pm 0.1}$} \\
& 1.80\% & 59.0$_{\pm 0.5}$ & 70.6$_{\pm 0.9}$ & 69.8$_{\pm 1.1}$ & 72.4$_{\pm 0.4}$ & 74.3$_{\pm 0.1}$ & 64.2$_{\pm 1.7}$ & 66.7$_{\pm 1.0}$ & 64.3$_{\pm 1.0}$ & \textbf{69.1$_{\pm 1.1}$} &  \\
& 3.60\% & 66.3$_{\pm 1.5}$ & 69.8$_{\pm 1.4}$ & 69.4$_{\pm 1.4}$ & 70.6$_{\pm 0.7}$ & 73.3$_{\pm 0.4}$ & 69.1$_{\pm 0.1}$ & 69.0$_{\pm 0.1}$ & 69.1$_{\pm 0.1}$  & \textbf{70.5$_{\pm 1.1}$} &  \\
\hline

\multirow{3}{*}{ Cora} 
& 1.30\% & 67.3$_{\pm 1.9}$ & 79.8$_{\pm 1.2}$ & 75.9$_{\pm 1.2}$ & 80.$_{\pm 0.4}$ &82.5$_{\pm 0.4}$ & 63.6$_{\pm 3.7}$ & 67.0$_{\pm 1.3}$ & 64.0$_{\pm 2.3}$  & \textbf{81.1$_{\pm 0.7}$} &  \multirow{3}{*}{81.2$_{\pm 0.2}$}\\
& 2.60\% &  67.6$_{\pm 3.5}$ & 80.1$_{\pm 0.6}$ & 75.7$_{\pm 0.9}$ & 81.7$_{\pm 0.5}$ & 83.6$_{\pm 0.3}$  & 72.8$_{\pm 1.1}$ & 73.4$_{\pm 1.0}$ & 73.2$_{\pm 1.2}$ & \textbf{81.6$_{\pm 0.6}$} &  \\
& 5.20\% & 67.7$_{\pm 2.2}$ & 79.3$_{\pm 0.3}$ & 76.0$_{\pm 0.3}$ & 81.6$_{\pm 0.8}$ & 82.8$_{\pm 0.7}$ &  76.8$_{\pm 0.1}$ & 76.8$_{\pm 0.1}$ & 76.7$_{\pm 0.1}$ &  \textbf{80.7$_{\pm 0.4 }$} &  \\
\hline

\multirow{3}{*}{\begin{tabular}[c]{@{}l@{}}Ogbn\\ arxiv\end{tabular}}
& 0.05\% & 58.6$_{\pm 0.4}$ & 59.2$_{\pm 1.1}$ & 61.3$_{\pm 0.5}$ & 65.5$_{\pm 0.7}$ & 65.5$_{\pm 0.6}$ & 47.1$_{\pm 3.9}$ & 52.4$_{\pm 1.8}$ & 47.2$_{\pm 3.0}$ &  \textbf{68.0$_{\pm  0.1 }$} &  \multirow{3}{*} {\textbf{71.4$_{\pm 0.1}$}}\\
& 0.25\% &  59.9$_{\pm 0.3}$ & 63.2$_{\pm 0.3}$ & 64.2$_{\pm 0.4}$ & 66.1$_{\pm 0.4}$ & 68.8$_{\pm 0.2}$ & 57.3$_{\pm 1.1}$ & 58.6$_{\pm 1.2}$ & 56.8$_{\pm 0.8}$ & \textbf{69.6$_{\pm 0.2}$} &  \\
& 0.50\% &  59.5$_{\pm 0.3}$ & 64.0$_{\pm 1.4}$ & 63.1$_{\pm 0.5}$ & 66.8$_{\pm 0.4}$ & 69.6$_{\pm 0.2}$ & 60.0$_{\pm 0.9}$ & 60.4$_{\pm 0.8}$ & 60.3$_{\pm 0.4}$ & \textbf{69.8$_{ \pm 0.1 }$} &  \\
\hline

\multirow{3}{*}{Flickr} 
& 0.10\% &  46.3$_{\pm 0.2}$ & 46.5$_{\pm 0.4}$ & 45.9$_{\pm 0.1}$ & 46.6$_{\pm 0.2}$ & 47.1$_{\pm 0.1}$ & 41.8$_{\pm 2.0}$ & 42.5$_{\pm 1.8}$ & 42.0$_{\pm 0.7}$ & \textbf{45.8 $_{\pm 0.1 }$} &  \multirow{3}{*}{47.2$_{\pm 0.1}$} \\
& 0.50\% &  45.9$_{\pm 0.1}$ & 47.1$_{\pm 0.1}$ & 45.0$_{\pm 0.1}$ & 47.0$_{\pm 0.1}$ & 47.0$_{\pm 0.2}$ & 44.0$_{\pm 0.4}$ & 43.9$_{\pm 0.9}$ & 43.2$_{\pm 0.1}$ & \textbf{46.8 $_{\pm 0.2 }$} &  \\
& 1.00\% &  44.6$_{\pm 0.1}$ & 47.1$_{\pm 0.1}$ & 45.0$_{\pm 0.1}$ & 47.1$_{\pm 0.1}$ & 47.3$_{\pm 0.3}$  & 44.6$_{\pm 0.2}$ & 44.4$_{\pm 0.6}$ & 44.1$_{\pm 0.4}$ & \textbf{47.3$_{\pm 0.1}$} &  \\
\hline

\multirow{3}{*}{Reddit} 
& 0.01\% &  88.2$_{\pm 0.2}$ & 88.0$_{\pm 1.8}$ & 88.4$_{\pm 0.4}$ & 89.7$_{\pm 0.2}$ & 91.1$_{\pm 0.4}$ & 46.1$_{\pm 4.4}$ & 53.1$_{\pm 2.5}$ & 46.6$_{\pm 2.3}$ &  \textbf{89.2 $_{\pm 0.2 }$} & \multirow{3}{*}{\textbf{93.9$_{\pm 0.0}$}} \\
& 0.20\% &  90.5$_{\pm 1.2}$ & 90.1$_{\pm 0.5}$ & 88.8$_{\pm 0.4}$ & 90.3$_{\pm 0.3}$ & 91.5$_{\pm 0.4}$ & 66.3$_{\pm 1.9}$ & 71.0$_{\pm 1.6}$ & 58.5$_{\pm 2.1}$ & \textbf{90.6 $_{\pm 0.2 }$} &  \\
& 3.00\% &  90.8$_{\pm 0.9}$ & OOM & 89.2$_{\pm 0.2}$ & 91.0$_{\pm 0.3}$ & 93.7$_{\pm 0.1}$ &  78.4$_{\pm 1.3}$ & 81.3$_{\pm 1.1}$ & 82.2$_{\pm 1.4}$ & \textbf{93.3 $_{\pm 0.1 }$} & \\ 
\hline

\end{tabular}
}
\caption{\small \textit{Standard Supervised Settings for Node Classification:} PLGC consistantly ourperforms the existing self-supervided methods, while achieving similar performance compared to the existing supervised methods even when they are trained using 100\% clean labels.
(OOM: out-of-memory error.)
}
\label{table_clean_single_node}
\end{table*}

\subsection{Node Classification}

\subsubsection{Single-Source: Clean-Label Fine-Tuning}
We first evaluate node classification under a single-source setting where all node labels are available during fine-tuning.
We compare PLGC with graph coarsening \cite{huang2021scaling}, graph coreset methods (Random, Herding \cite{herding}, K-Center \cite{kcenter}), dataset condensation approaches (DC-Graph \cite{zhao2020dataset,gcond}), and state-of-the-art supervised graph condensation methods including GCond-X, GCond \cite{gcond}, SFGC \cite{sfgc}, and GEOM \cite{geom}.

Supervised condensation methods follow a two-stage pipeline: 
(i) a condensation stage that synthesizes a labeled condensed graph using a GCN, and 
(ii) supervised training of a GNN on the condensed graph for downstream prediction.
In contrast, PLGC produces an unlabeled condensed graph via self-supervised learning.
We therefore fine-tune only the prediction head ($\phi$) using clean labels while freezing the backbone.

Table~\ref{table_clean_single_node} reports node classification accuracy across five datasets and three compression ratios $r \in (0,1)$, where the condensed graph size is $N' = rN$.
Here, $N$ denotes the total number of nodes in transductive settings and the number of training nodes in inductive settings.
PLGC consistently outperforms existing self-supervised methods and achieves performance comparable to supervised condensation approaches under clean-label fine-tuning.

\begin{wraptable}{r}{0.6\textwidth}
\vspace{-1em}
\centering
\resizebox{0.59\textwidth}{!}{
\begin{tabular}{l|c|cccc}
\hline
\multicolumn{1}{c}{\begin{tabular}[c]{@{}c@{}}3-Shot \\ Classification\end{tabular}} & r & \textbf{GCond} & CTGC & \textbf{\begin{tabular}[c]{@{}c@{}}PLGC \\ (Ours)\end{tabular}} & \textbf{\begin{tabular}[c]{@{}c@{}}Whole \\ Dataset\end{tabular}} \\
\hline
Cora & 2.60\% & 55.6 $\pm$  3.0 & 70.2$\pm$2.5 & \textbf{73.4 $\pm$ 3.2} & 65.9 $\pm$ 4.3 \\
Citeseer & 1.80\% & 55.9 $\pm$ 0.8 & 63.2 $\pm$ 1.8 & \textbf{63.4 $\pm$ 3.7} & 57.2 $\pm$ 2.4 \\
Ogbn-Arxiv & 0.05\% & 44.6 $\pm$ 0.4 & 48.1 $\pm$ 0.5 & \textbf{48.7 $\pm$ 0.3} & 46.3 $\pm$ 3.7 \\
Reddit & 0.1\% & 70.0 $\pm$ 1.2 & 86.4 $\pm$ 0.2 & \textbf{88.3 $\pm$ 0.2} & 85.4 $\pm$ 0.3 \\
\hline
\end{tabular}
}
\caption{\small Compression with CTGC, the only existing self-supervised graph condensation method for 3-shot classification tasks.}
\label{table_ssgc_comparison}
\end{wraptable}

\begin{figure*}[t]
  \centering
  \begin{subfigure}{0.19\textwidth}
    \centering
    \includegraphics[height=2.3cm, width=\linewidth]{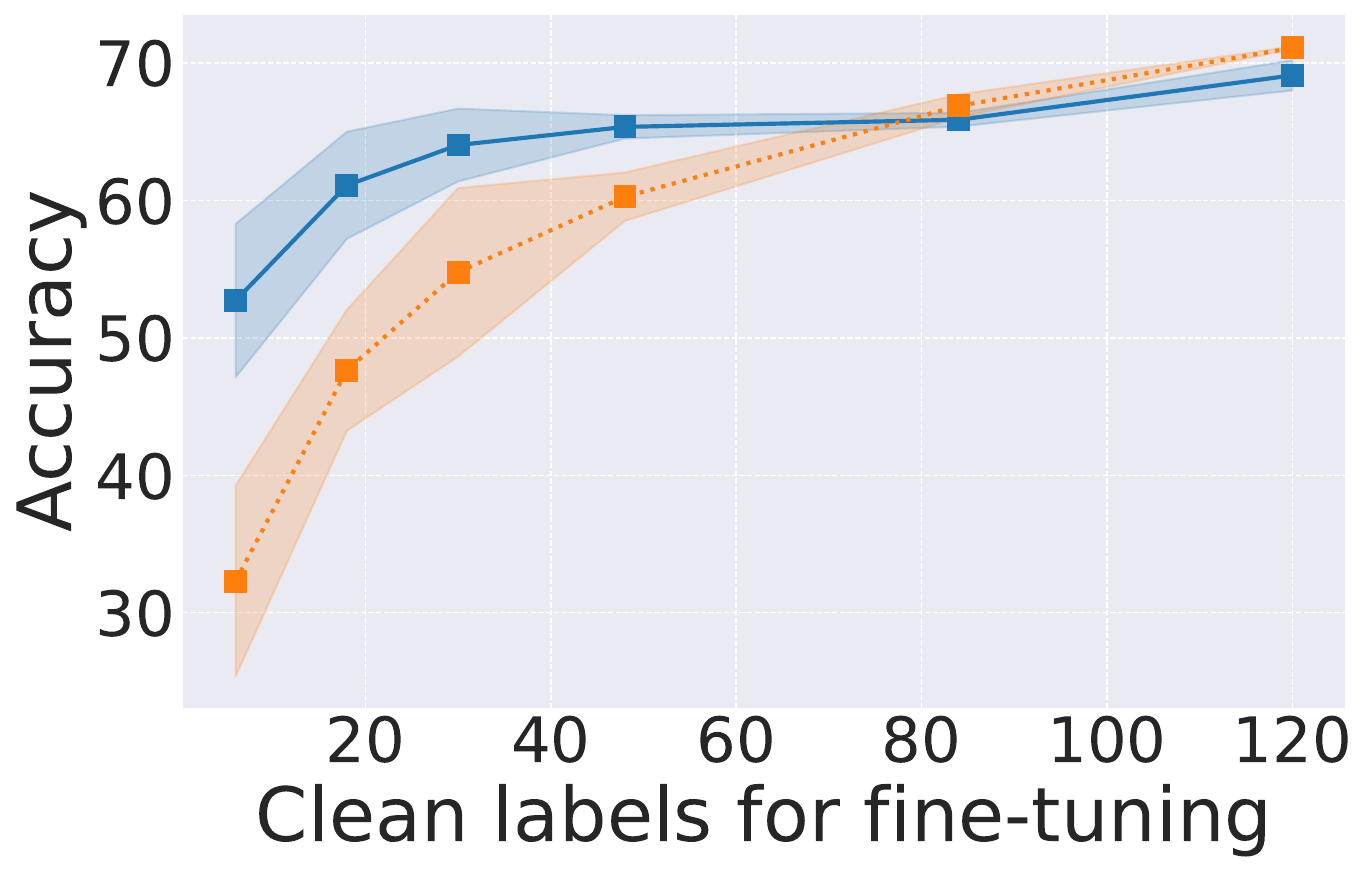}
    \caption{Citeseer}
    \label{fig:abl1_citeseer}
  \end{subfigure}
  \hspace{0pt} % No space between figures
  \begin{subfigure}{0.19\textwidth}
    \centering
    \includegraphics[height=2.3cm, width=\linewidth]{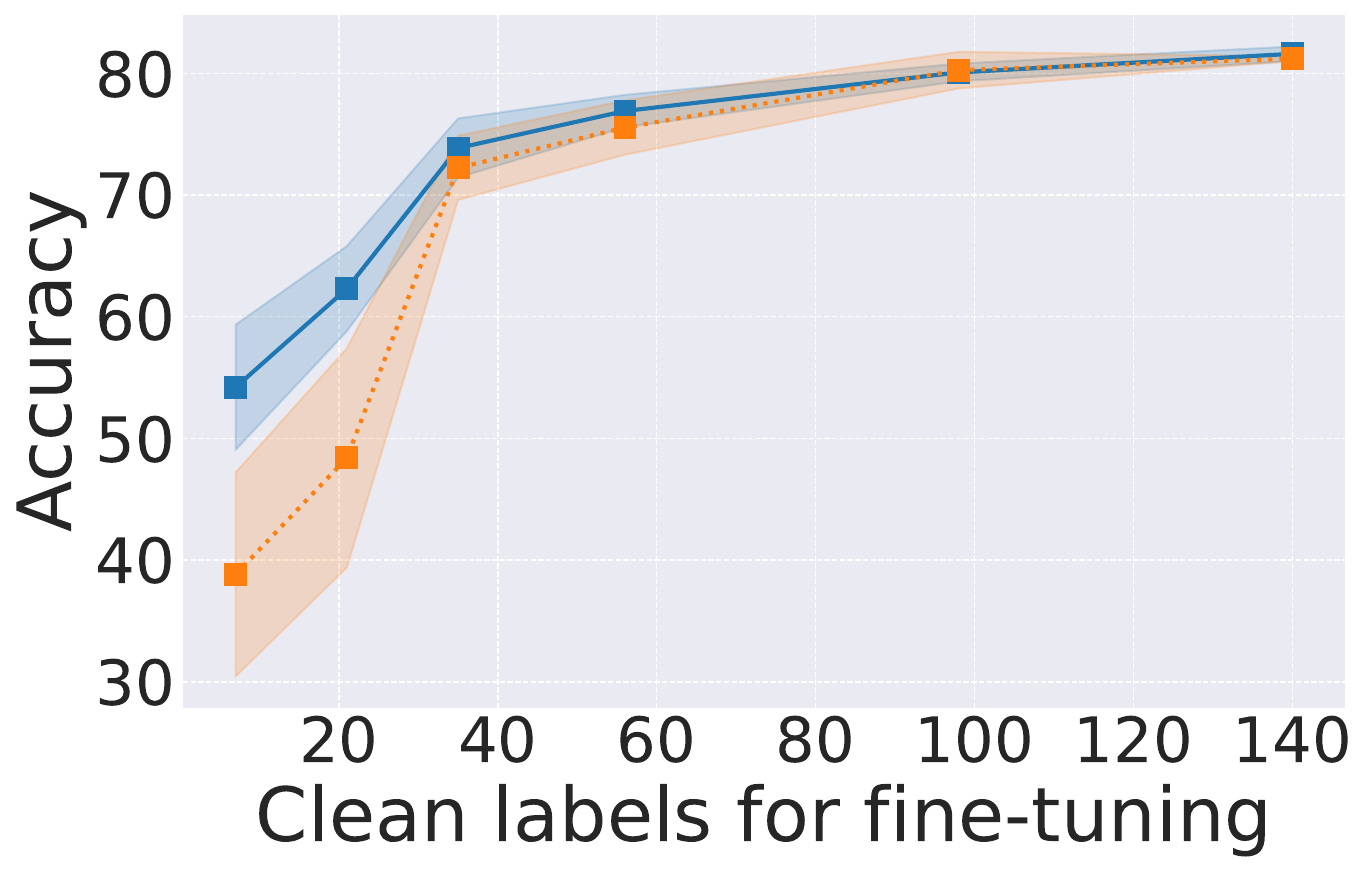}
    \caption{Cora}
    \label{fig:abl1_cora}
  \end{subfigure}
  \hspace{0pt} % No space between figures  
  \begin{subfigure}{0.19\textwidth}
    \centering
    \includegraphics[height=2.3cm, width=\linewidth]{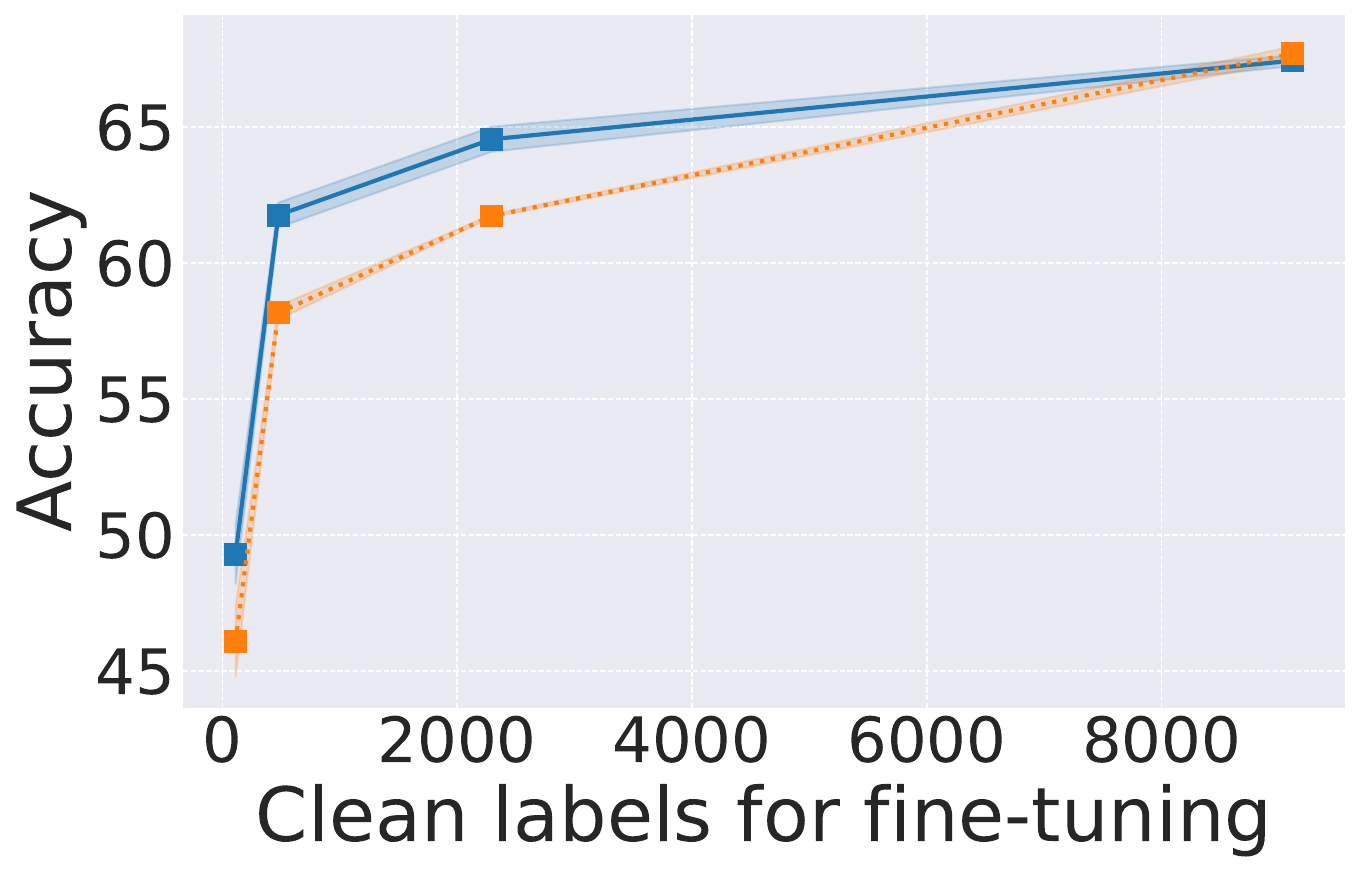}
    \caption{\small Ogbn-arxiv}
    \label{fig:abl1_ogbn}
  \end{subfigure}
  \begin{subfigure}{0.19\textwidth}
    \centering
    \includegraphics[height=2.3cm, width=\linewidth]{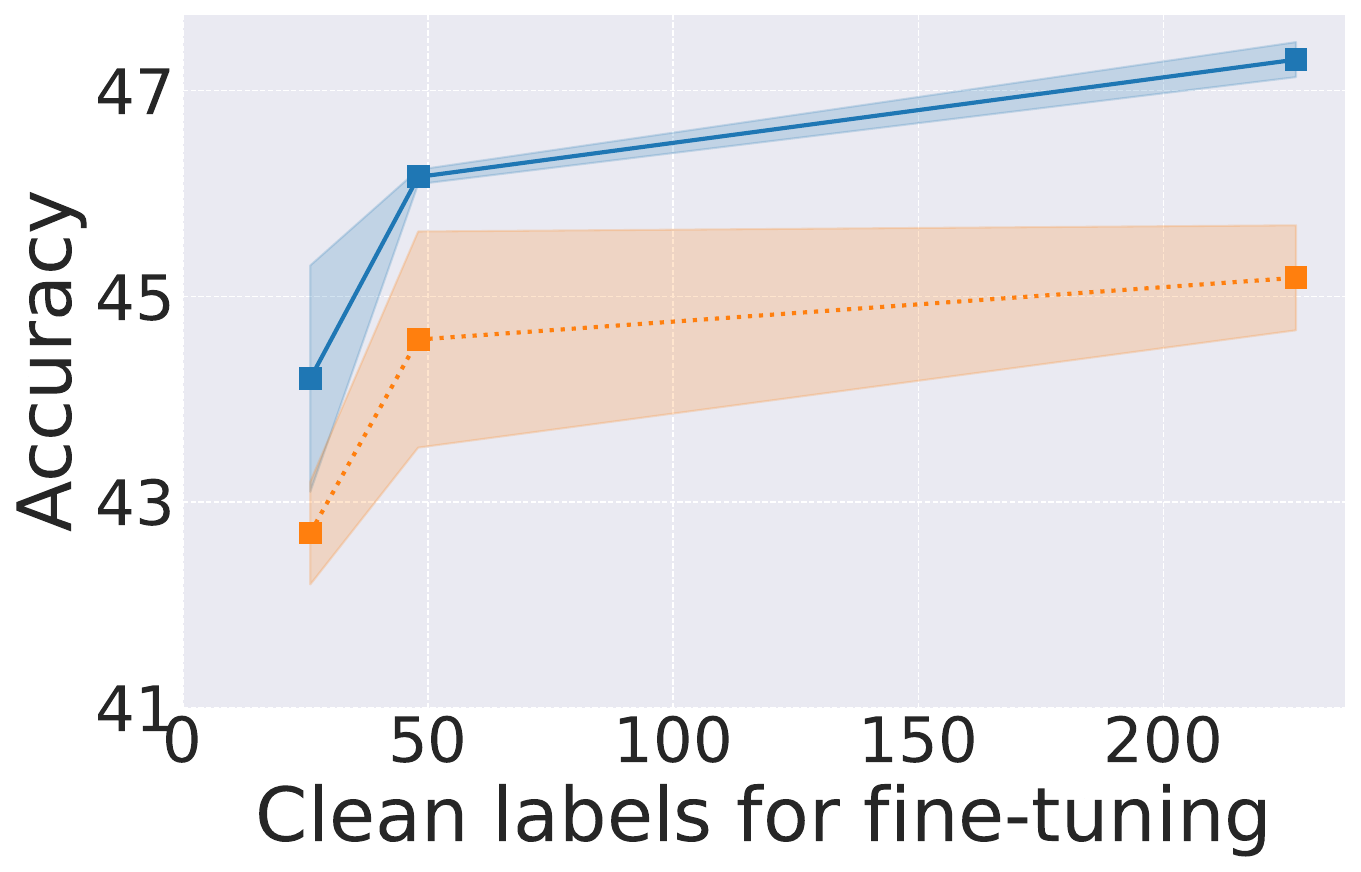}
    \caption{Flickr}
    \label{fig:abl1_flickr}
  \end{subfigure}
  \hspace{0pt} % No space between figures in second row
  \begin{subfigure}{0.19\textwidth}
    \centering
    \includegraphics[height=2.3cm, width=\linewidth]{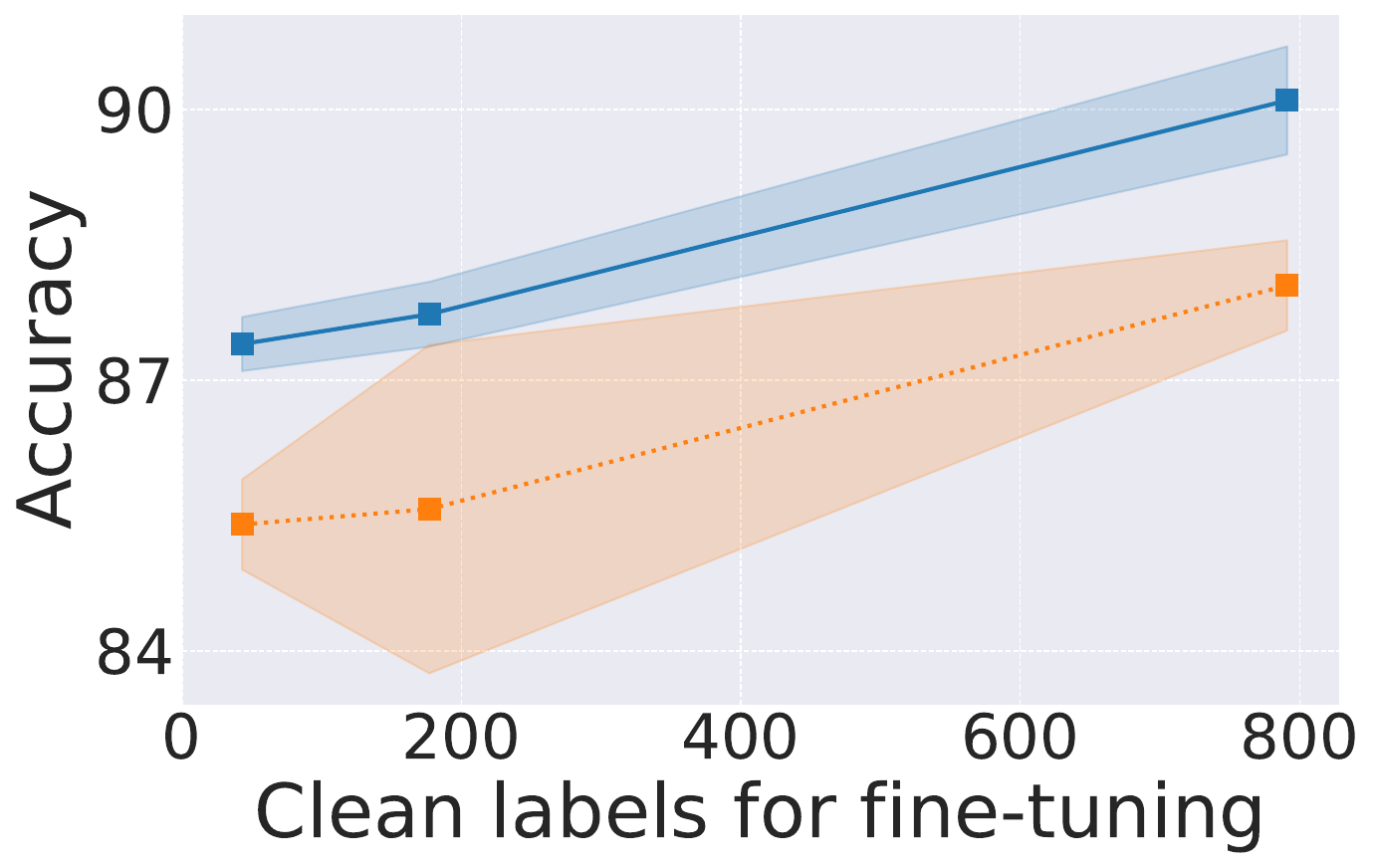}
    \caption{Reddit}
    \label{fig:abl1_reddit}
  \end{subfigure}  

    %\vspace{1em} % Add some space between the two rows
  \begin{subfigure}{0.6\textwidth}
    \centering
    \includegraphics[width=0.6\linewidth]{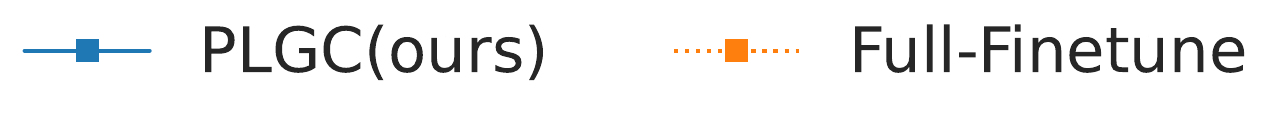}
    \label{fig_node_legend_abl}
  \end{subfigure}
  \caption{\small \textit{Single-Source with Label Noise for Node classification:} By leveraging the self-supervised condensed graphs, PLGC achieves better performance at lesser number of clean labelled nodes during funetuning stage, and demonstrates faster convergence compared to the `full-finetune' baselines.
  } 
  \label{fig_abl1}
\end{figure*}

\subsubsection{Single-Source: Few-Shot Self-Supervised Fine-Tuning}
We next evaluate PLGC under a few-shot setting, where only a small number of labeled nodes per class are available during fine-tuning.
Specifically, we consider a 3-shot setup, where the training graph is obtained from a single source for pseudo-label learning, and only three labeled nodes per class are used for downstream adaptation.

Table~\ref{table_ssgc_comparison} compares PLGC with CTGC, the only existing self-supervised graph condensation method applicable to node classification, on Cora, Citeseer, Ogbn-Arxiv, and Reddit.
PLGC consistently outperforms CTGC across all datasets.

Figure~\ref{fig_abl1} further shows node classification accuracy as the number of labeled nodes increases.
We compare against a \emph{Full Fine-Tune} baseline, where a GNN is trained directly on labeled nodes without condensation.
PLGC achieves higher accuracy and faster convergence, demonstrating improved label efficiency and stronger representations.

\begin{wraptable}{r}{0.5\textwidth}
\vspace{-1em}
\centering
\resizebox{0.46\textwidth}{!}{
\begin{tabular}{l|ccc}
\hline
Datasets & r & \# Clean Labels & \# Training Labels \\
\hline
Citeseer & 1.80\% & 18 & 120 \\
Cora & 2.60\% & 21 & 140 \\
Ogbn-arxiv & 0.05\% & 113 & 90,941 \\
Flickr & 1.00\% & 26 & 44,625 \\
Reddit & 1.00\% & 43 & 153,932
\\ \hline
\end{tabular}
}
\caption{\small Compression ratios and available clean nodes for finetuning for `Single-source with Label noise' setup.}
\label{table_source1}
\vspace{-0.5em}
\end{wraptable}
\subsubsection{Single-Source: Label Noise}
We evaluate robustness to label noise under a single-source setting by corrupting a fraction of node labels during condensation.
Fine-tuning is performed using a small subset of clean-labeled nodes, with at least one node per class.
We compare PLGC with GCond and SFGC.

For PLGC, the prediction head is randomly initialized and trained solely on clean labels.
In contrast, GCond and SFGC reuse their prediction heads trained during condensation and fine-tune them using clean labels.
Compression ratios and fine-tuning label counts are summarized in Table~\ref{table_source1}.

\begin{figure*}[!t]
  \centering
  \begin{subfigure}{ 0.19\textwidth}
    \centering
    \includegraphics[ height=2.3cm,width=\linewidth]{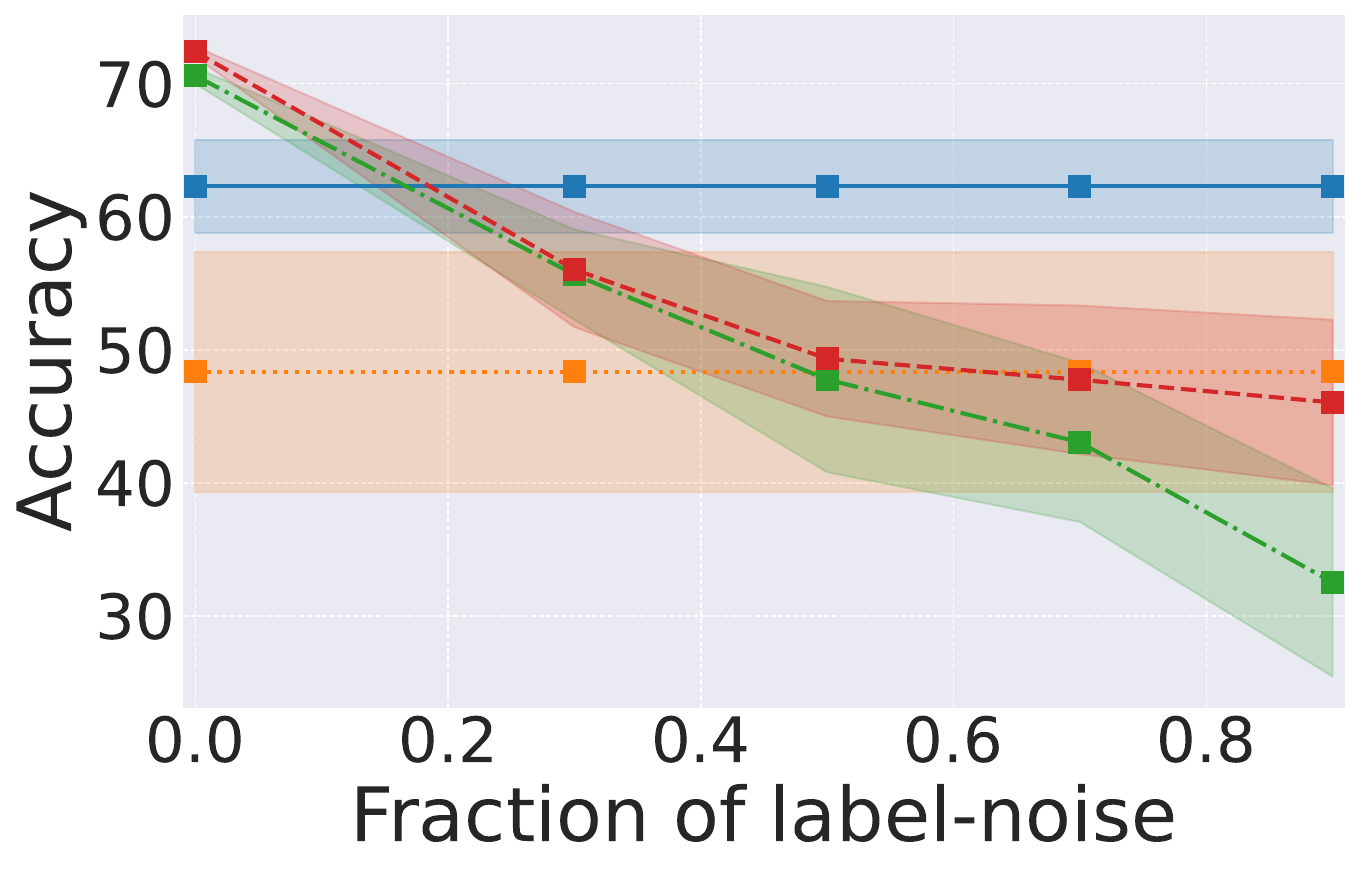}
    \caption{Citeseer}
    \label{fig:node_citeseer_1}
  \end{subfigure}
  \hspace{0pt} % No space between figures
  \begin{subfigure}{0.19\textwidth}
    \centering
    \includegraphics[ height=2.3cm, width=\linewidth]{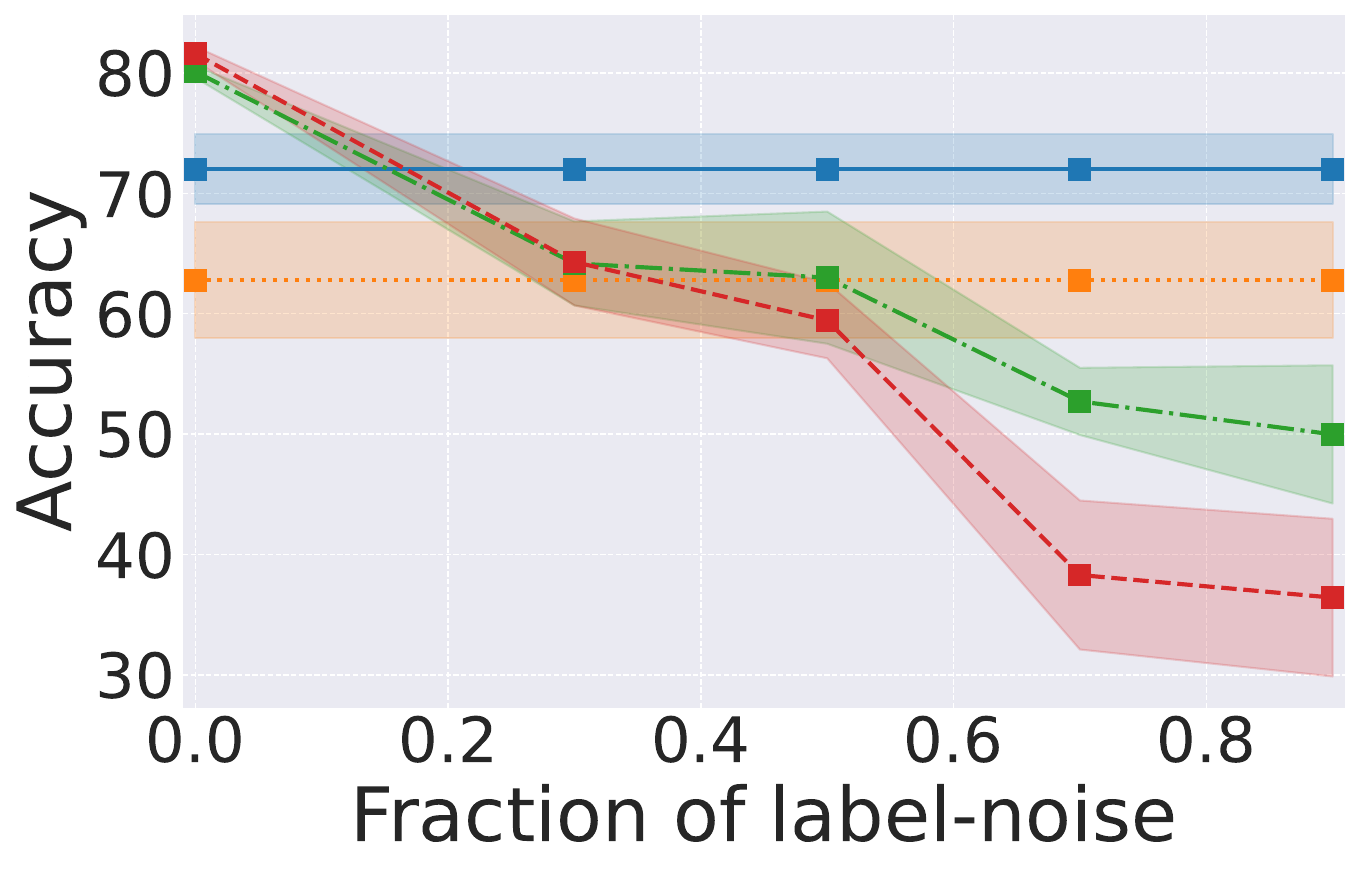}
    \caption{Cora}
    \label{fig:node_cora_1}
  \end{subfigure}
  \hspace{0pt} % No space between figures  
  \begin{subfigure}{0.19\textwidth}
    \centering
    \includegraphics[height=2.3cm,width=\linewidth]{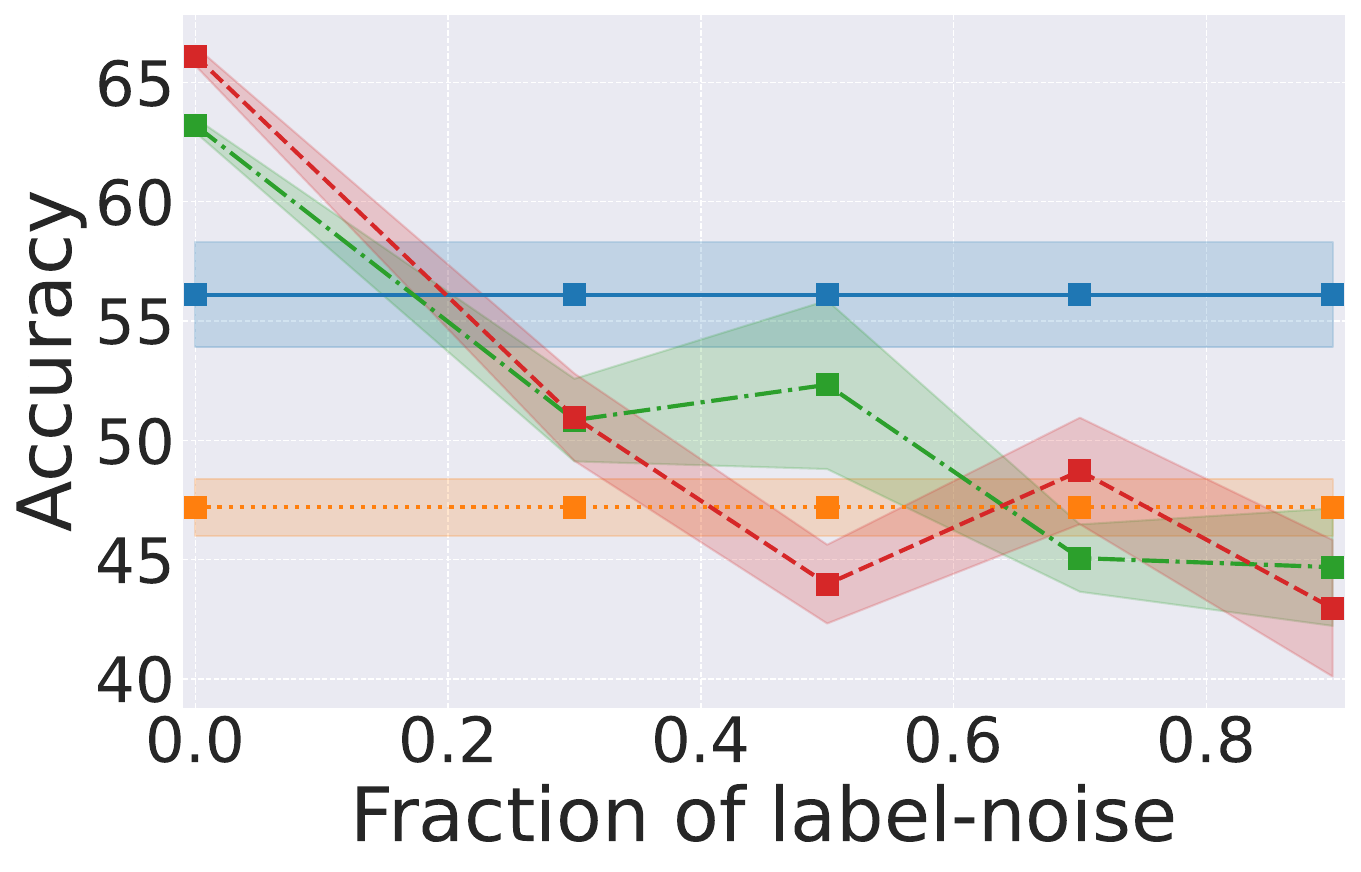}
    \caption{Ogbn-arxiv}
    \label{fig:node_ogbn_1}
  \end{subfigure}
  \begin{subfigure}{0.19\textwidth}
    \centering
    \includegraphics[height=2.3cm,width=\linewidth]{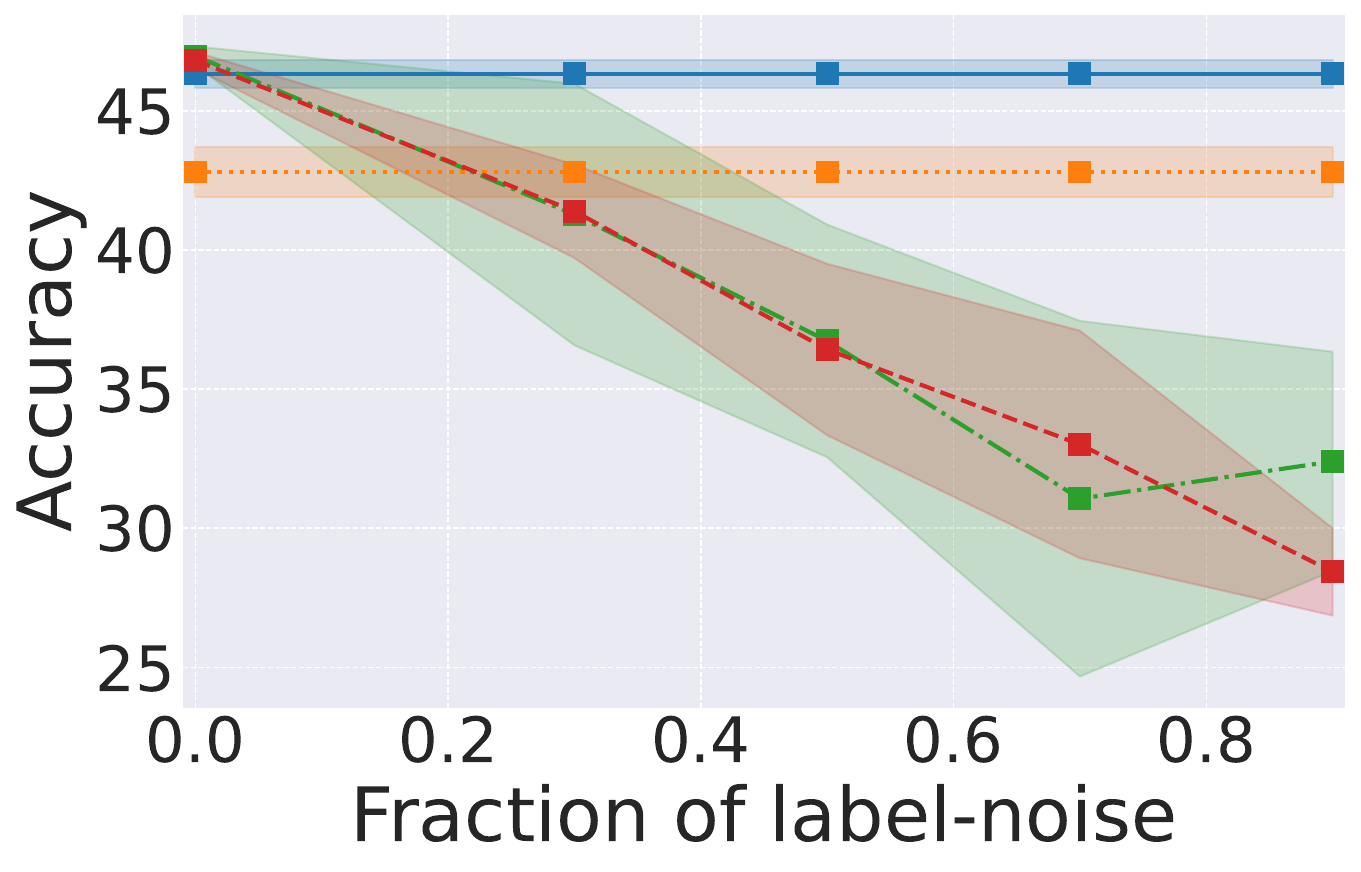}
    \caption{Flickr}
    \label{fig:node_flickr_1}
  \end{subfigure}
  \hspace{0pt} % No space between figures in second row
  \begin{subfigure}{0.19\textwidth}
    \centering
    \includegraphics[height=2.3cm, width=\linewidth]{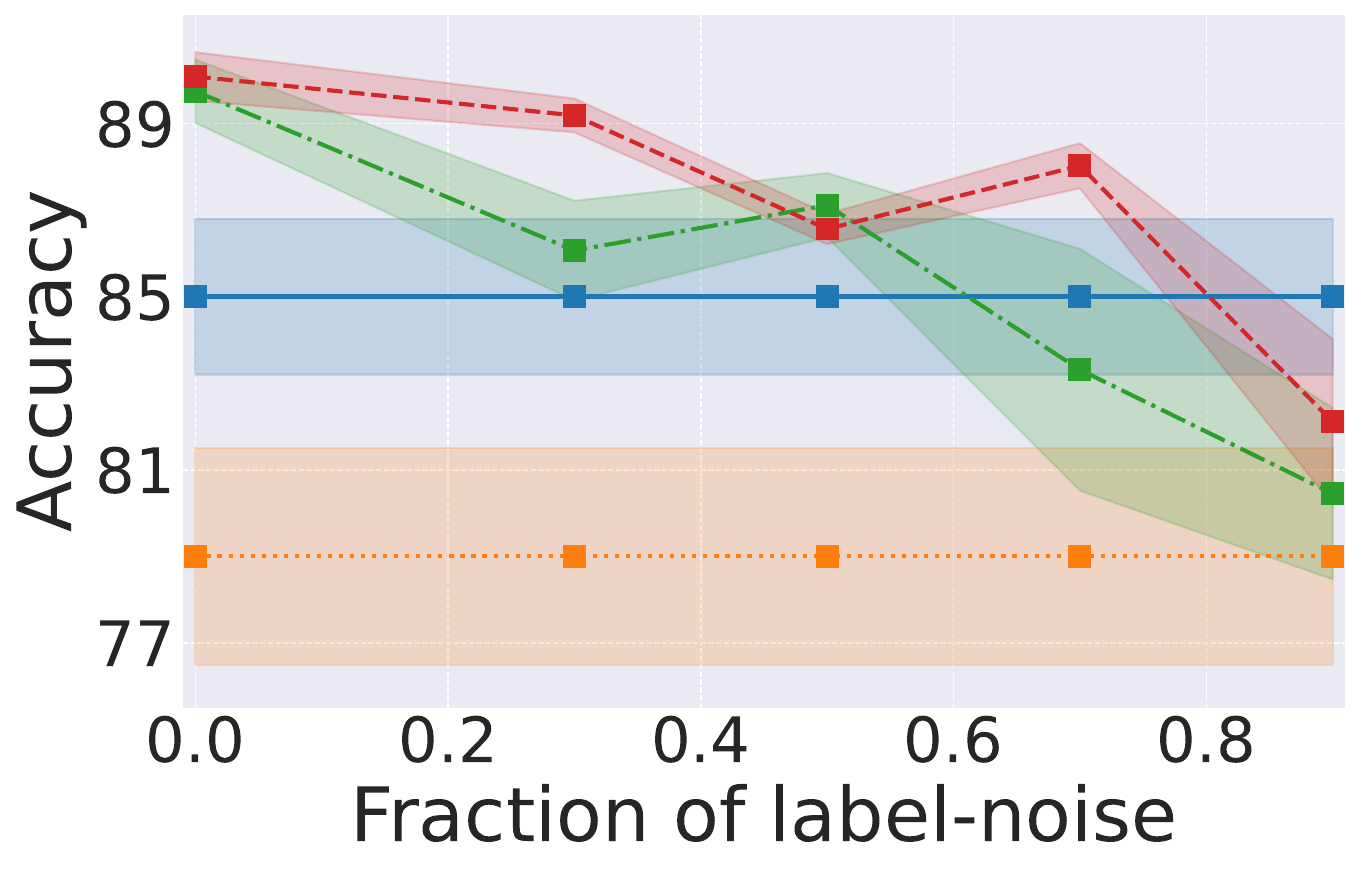}
    \caption{Reddit}
    \label{fig:node_reddit_1}
  \end{subfigure}  

    %\vspace{1em} % Add some space between the two rows
  \begin{subfigure}{0.6\textwidth}
    \centering
    \includegraphics[ width=\linewidth]{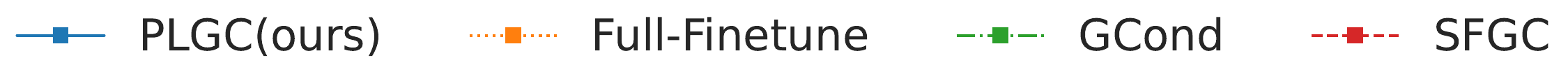}
    \label{fig_node_legend_1}
  \end{subfigure}
\vspace{-1.5em}
\caption{\small \textit{Single-Source with Label Noise for Node classification:} The superiority of our PLGC method is clearly visible as node-level label noises are increased, producing a robust, stable accuracy compared to the supervised GCond and SFGC methods.}
  \vspace{-1.5em}
  \label{fig_node_source1}
\end{figure*}

Figure~\ref{fig_node_source1} shows that supervised condensation methods perform well at $0\%$ noise but degrade rapidly as noise increases, eventually reaching a saturation point where noisy condensed graphs limit downstream performance.
While large datasets (e.g., Reddit) retain sufficient clean labels under moderate noise, this assumption is often unrealistic in practice.
PLGC remains robust across noise levels and significantly outperforms supervised methods at higher noise rates, despite using only $\approx 1$ clean-labeled node per class during fine-tuning.

\begin{wraptable}{r}{0.5\textwidth}
\vspace{-1em}
\centering
\resizebox{0.46\textwidth}{!}{
\begin{tabular}{l|ccc}
\hline
Datasets & r & \# Clean Labels & \# Training Labels \\
\hline
Citeseer & 1.80\% & 24 & 120 \\
Cora & 2.60\% & 28 &140 \\
Ogbn-arxiv & 0.05\% & 165 &90,941 \\
Flickr & 1.00\% & 33 & 44,625 \\
Reddit & 1.00\% & 123 & 153,932
\\ \hline
\end{tabular}
}
\caption{\small Compression ratios and clean nodes during finetuning for `Multi-source with Label noise' setup.}
\label{table_source3}
\vspace{-2em}
\end{wraptable}

\begin{figure*}[htbp]
  \centering
  % First Row
  \begin{subfigure}[b]{0.19\textwidth}
    \centering
    \includegraphics[height=2.3cm, width=\linewidth]{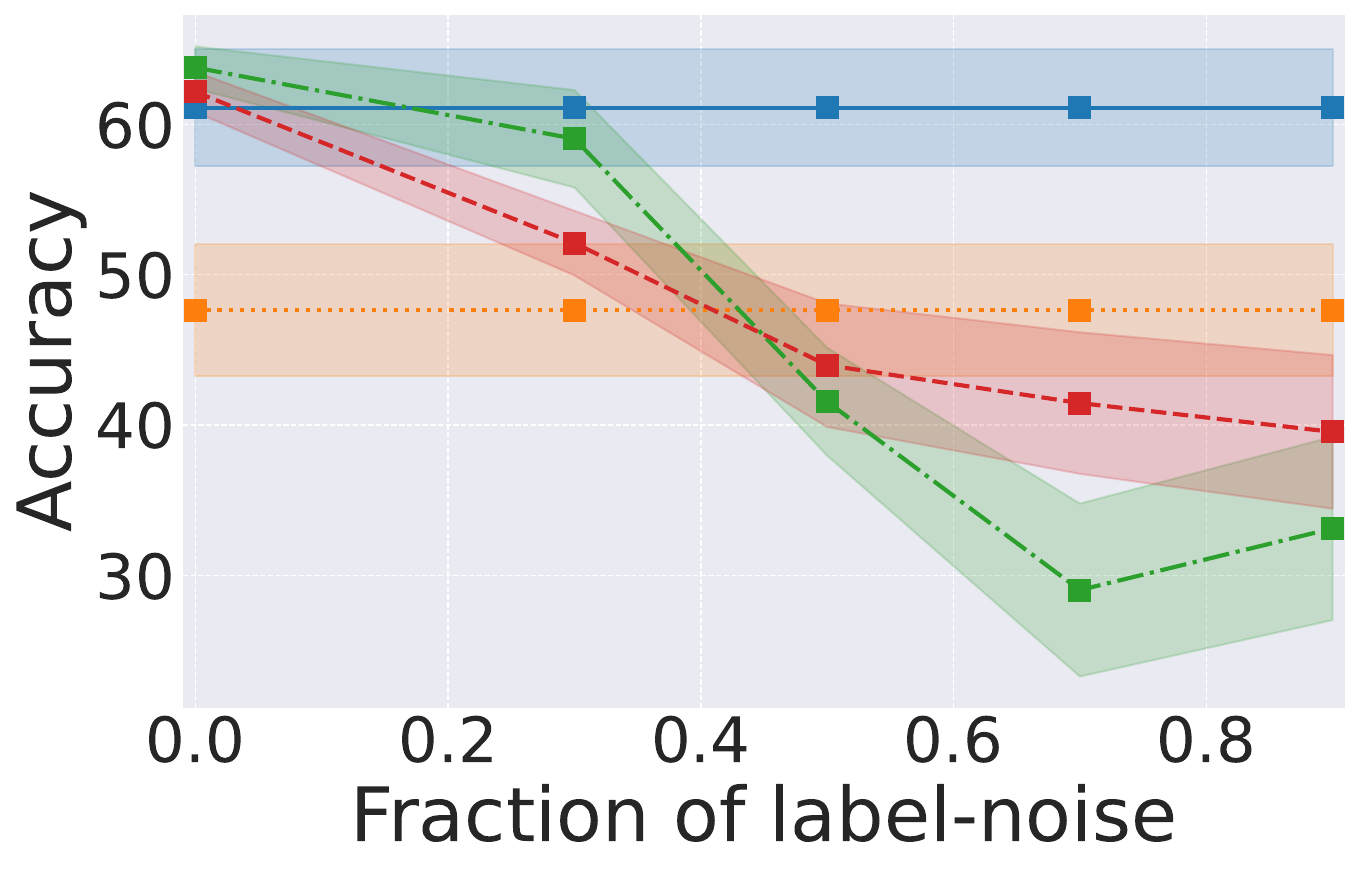}
    \caption{ Citeseer}\label{fig:node_citeseer_3}
  \end{subfigure}
  \begin{subfigure}[b]{0.19\textwidth}
    \centering
    \includegraphics[height=2.3cm, width=\linewidth]{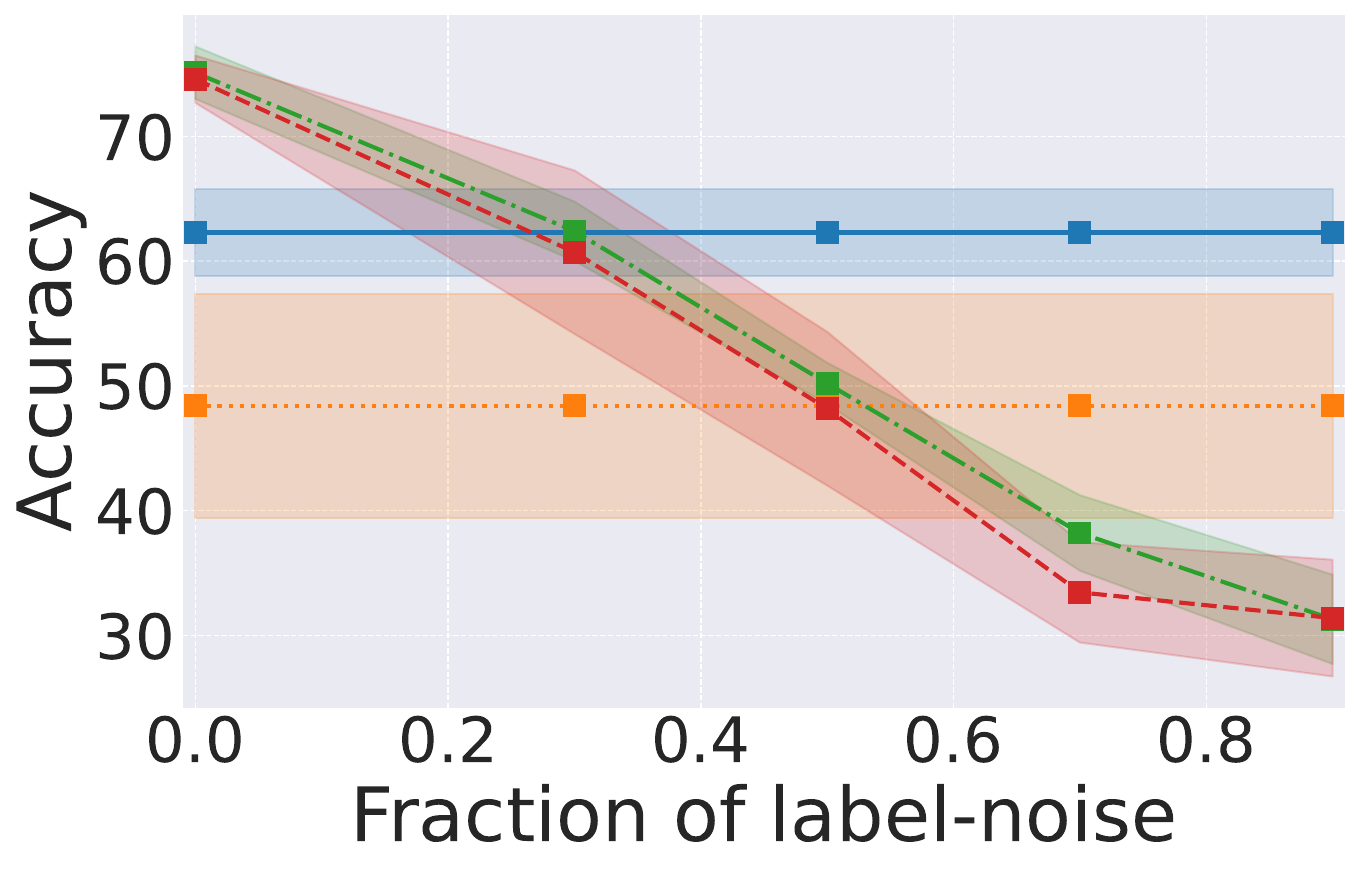}
    \caption{Cora}\label{fig:node_cora_3}
  \end{subfigure}
  \begin{subfigure}[b]{0.19\textwidth}
    \centering
    \includegraphics[height=2.3cm,width=\linewidth]{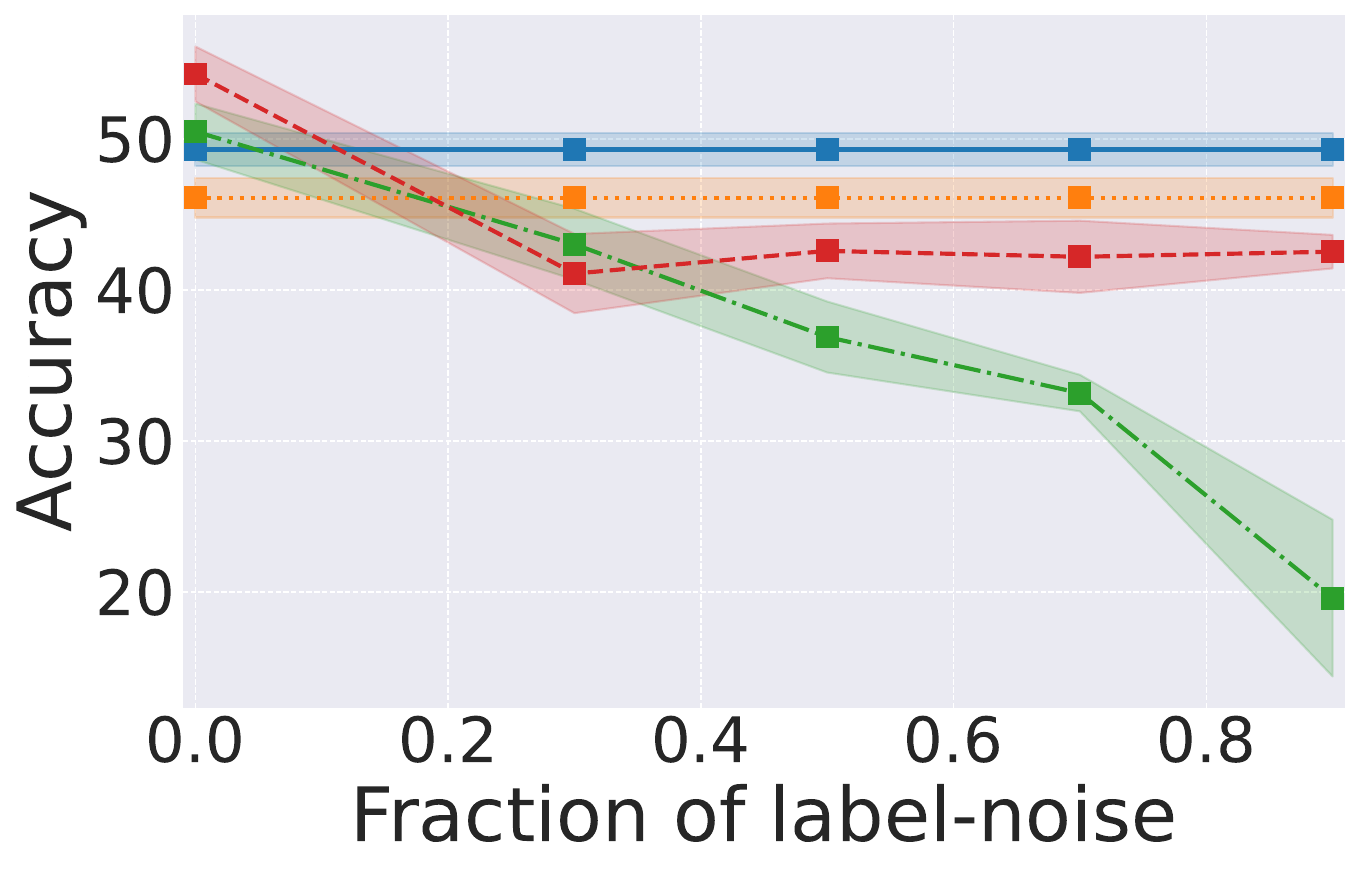}
    \caption{Ogbn-arxiv} \label{fig:node_ogbn_3}
  \end{subfigure}
  \begin{subfigure}[b]{0.19\textwidth}
    \centering
    \includegraphics[height=2.3cm, width=\linewidth]{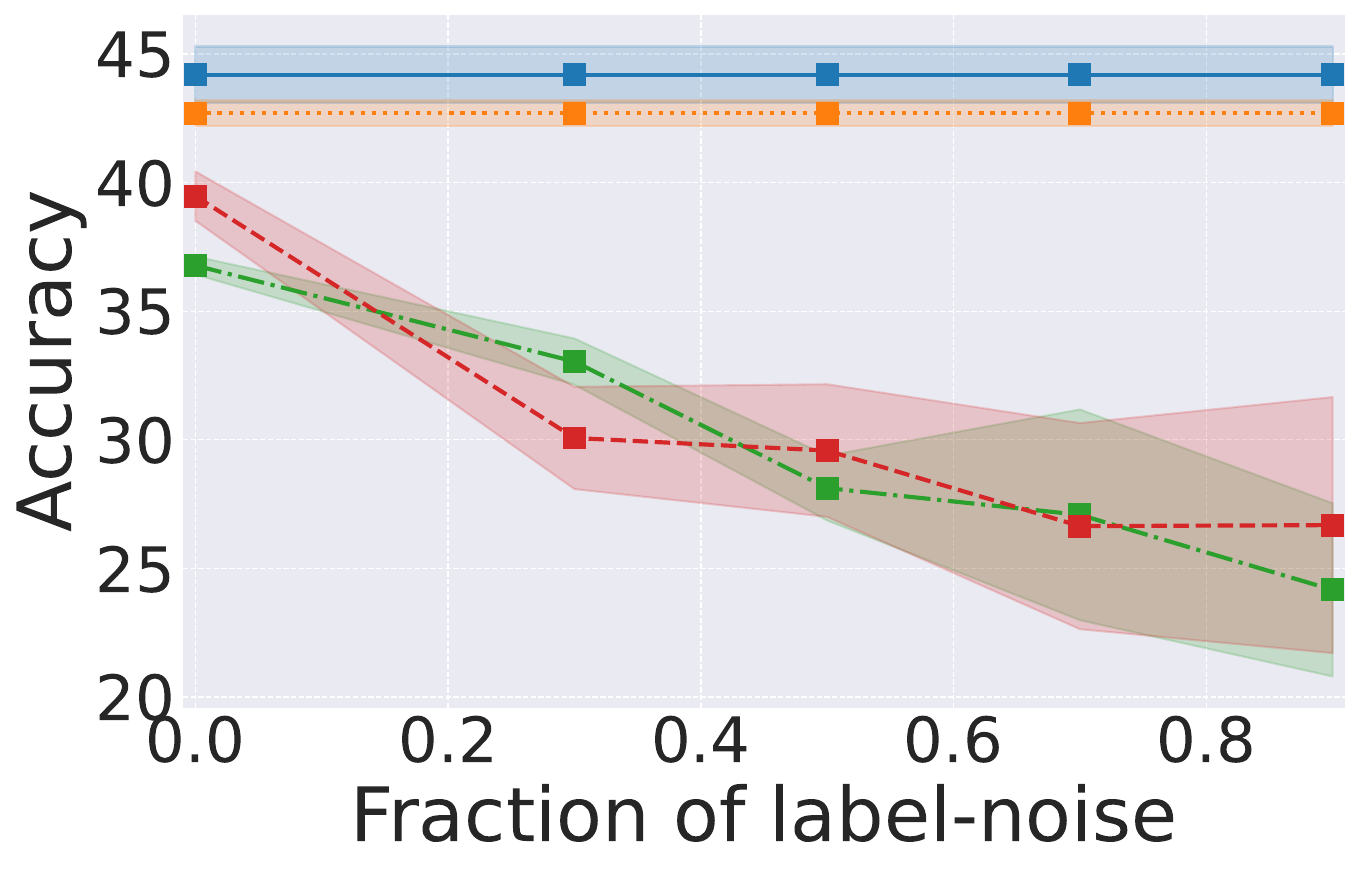}
    \caption{Flickr} \label{fig:node_flickr_3}
  \end{subfigure}
  \begin{subfigure}[b]{0.19\textwidth}
    \centering
    \includegraphics[height=2.3cm, width=\linewidth]{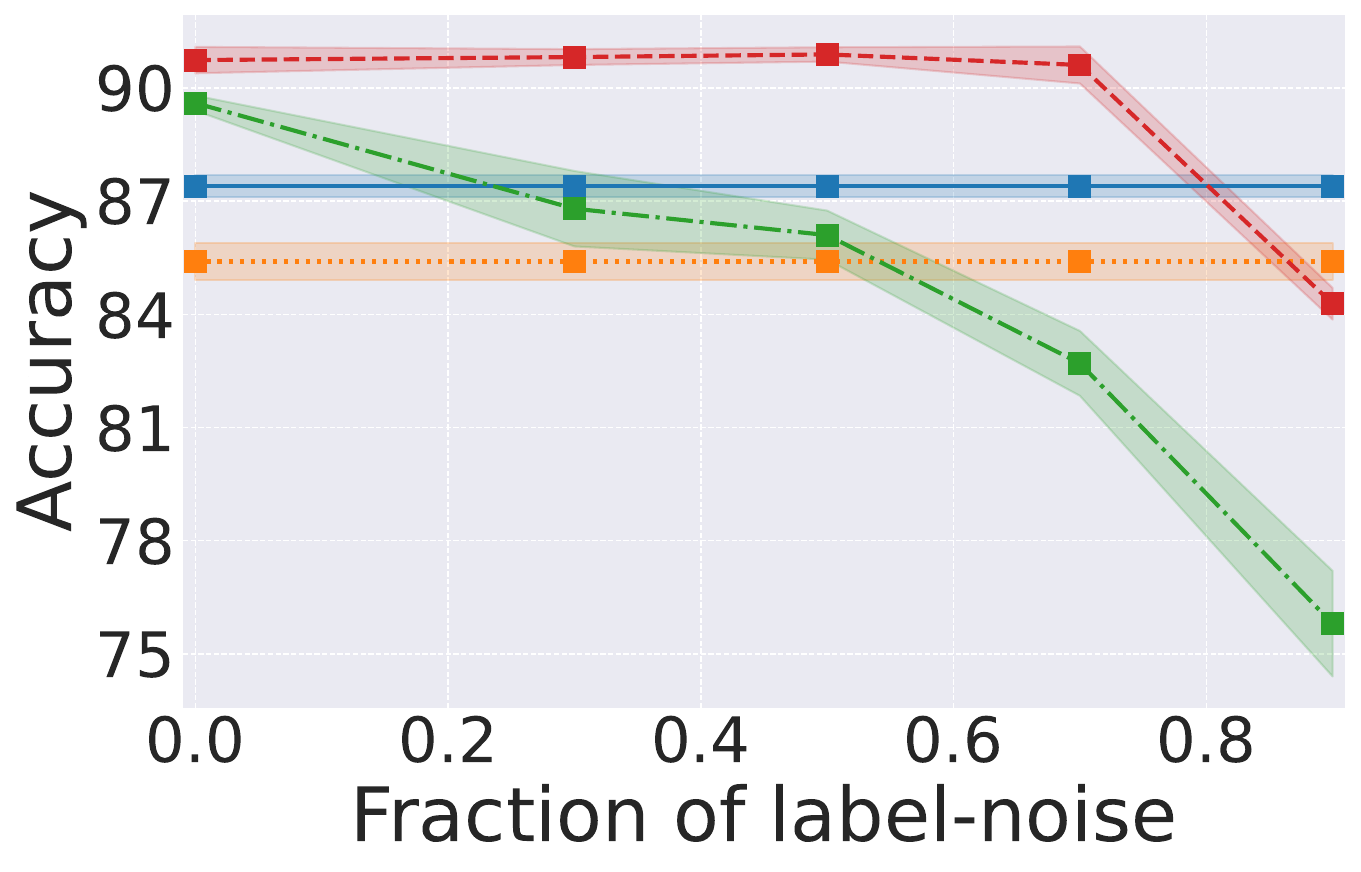}
    \caption{Reddit} \label{fig:node_reddit_3}
  \end{subfigure}
  
  \begin{subfigure}[b]{0.6\textwidth}
    \centering
    \includegraphics[width=\linewidth]{./images/swav_results_noRange/legend.pdf}
    %\caption{}\label{fig_node_legend_3}
  \end{subfigure}
  \caption{\small \textit{Multi-Source with Label Noise for Node classification:} The superiority of our PLGC method is clearly visible as node-level label noises are increased, producing a robust, stable accuracy compared to the supervised GCond and SFGC methods.}
  \label{fig_node_source3}
  \vspace{-1em}
\end{figure*}

\subsubsection{Multi-Source: Label Noise}
We extend the evaluation to a multi-source setting with three subgraphs, where label noise is introduced independently within each source during condensation.
Fine-tuning uses a small set of clean-labeled nodes pooled across all sources.

Condensed graphs are generated separately from each source and used to train a downstream GNN.
PLGC initializes the prediction head randomly and trains it only on clean labels, whereas GCond and SFGC fine-tune their pre-trained heads.
Experimental configurations are listed in Table~\ref{table_source3}.

Figure~\ref{fig_node_source3} shows that supervised methods degrade rapidly as noise increases, particularly on smaller graphs.
In contrast, PLGC consistently achieves higher accuracy under high noise, demonstrating strong robustness and generalization in heterogeneous multi-source environments.

\subsection{Link Prediction}

We evaluate condensed graphs on downstream link prediction under a single-source noisy setting.
In transductive settings, edges are split into training, validation, and test sets in a 1:1:2 ratio.
In inductive settings, the condensed graphs generated for node classification are reused.
During fine-tuning, only the prediction head ($\phi$) is updated, while the GNN backbone ($\text{GNN}_\theta$) remains frozen.
Performance is evaluated using AUROC \cite{auroc_2006}.

Figure~\ref{fig_link_source1} compares PLGC with GCond, SFGC, and a \emph{Full Fine-Tune} baseline.
PLGC achieves superior performance under increasing label noise for all datasets except Flickr.
Notably, both PLGC and supervised condensation methods (under clean labels) outperform full fine-tuning on most datasets, highlighting the importance of node-level optimization even for link prediction.
Supervised methods remain sensitive to label noise due to reliance on node-level supervision during condensation, whereas PLGC maintains stable performance through self-supervised learning.

\subsection{Key Observations and Summary}

\noindent\textbf{(i) Robustness to Label Noise.}
PLGC consistently outperforms supervised condensation methods under moderate to high label noise, demonstrating strong robustness to noisy supervision.

\noindent\textbf{(ii) Label Efficiency.}
PLGC achieves competitive performance using only a few labeled nodes per class during fine-tuning, substantially reducing annotation requirements.

\noindent\textbf{(iii) Multi-Source Generalization.}
Across heterogeneous sources, PLGC maintains stable performance, whereas supervised methods degrade rapidly due to inconsistent or noisy labels.

\noindent\textbf{(iv) Transferability Across Tasks.}
Condensed graphs produced by PLGC generalize effectively from node classification to link prediction, underscoring the quality of the learned representations.

Overall, these results demonstrate that PLGC provides a robust and label-efficient framework for graph dataset condensation under noisy and multi-source settings.

\begin{figure*}[!ht]
  \centering
  \begin{subfigure}{0.19\textwidth}
    \centering
    \includegraphics[height=2.3cm, width=\linewidth]{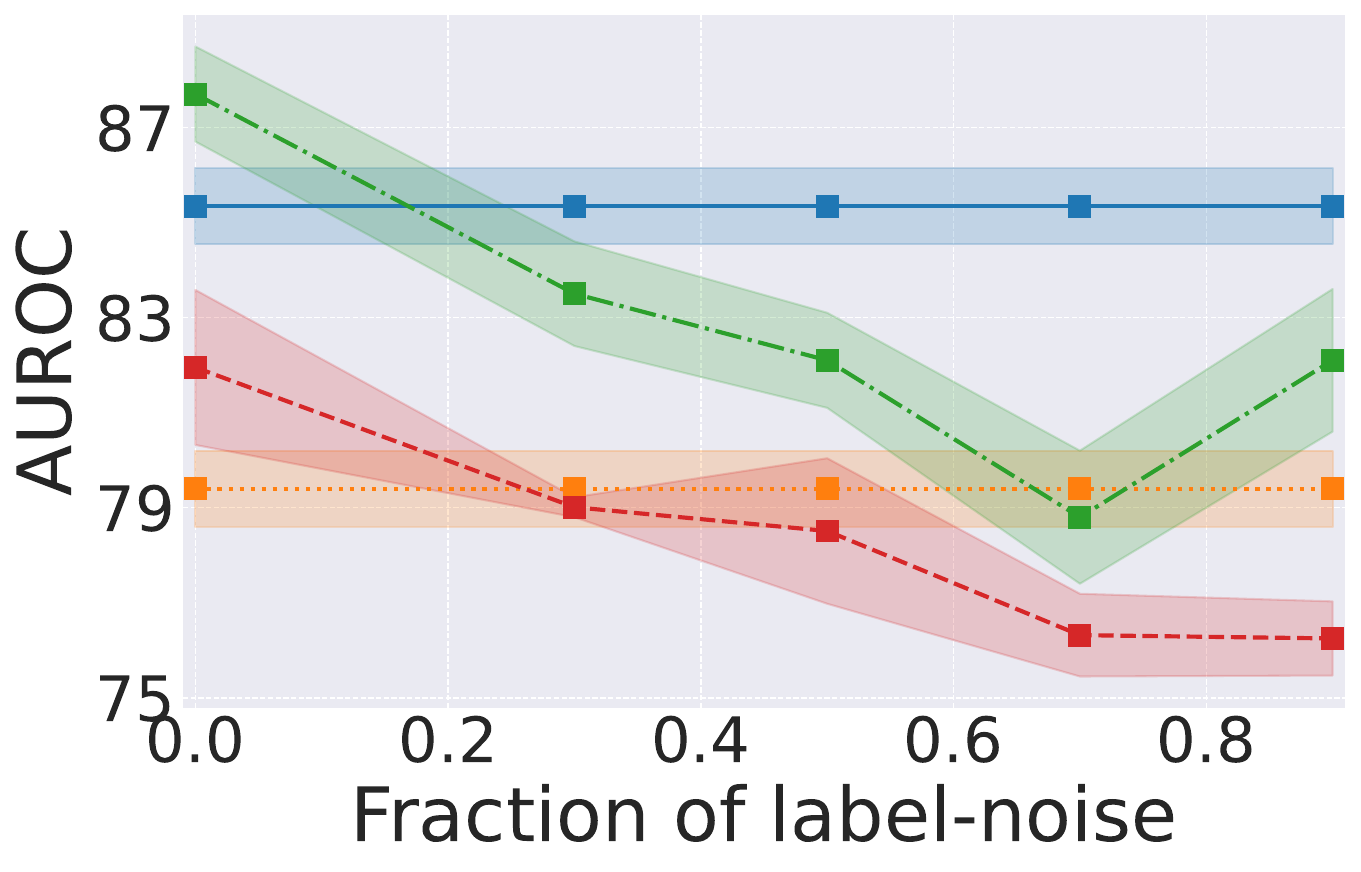}
    \caption{Citeseer}
    \label{fig:fig1}
  \end{subfigure}
  \hspace{0pt} % No space between figures
  \begin{subfigure}{0.19\textwidth}
    \centering
    \includegraphics[height=2.3cm, width=\linewidth]{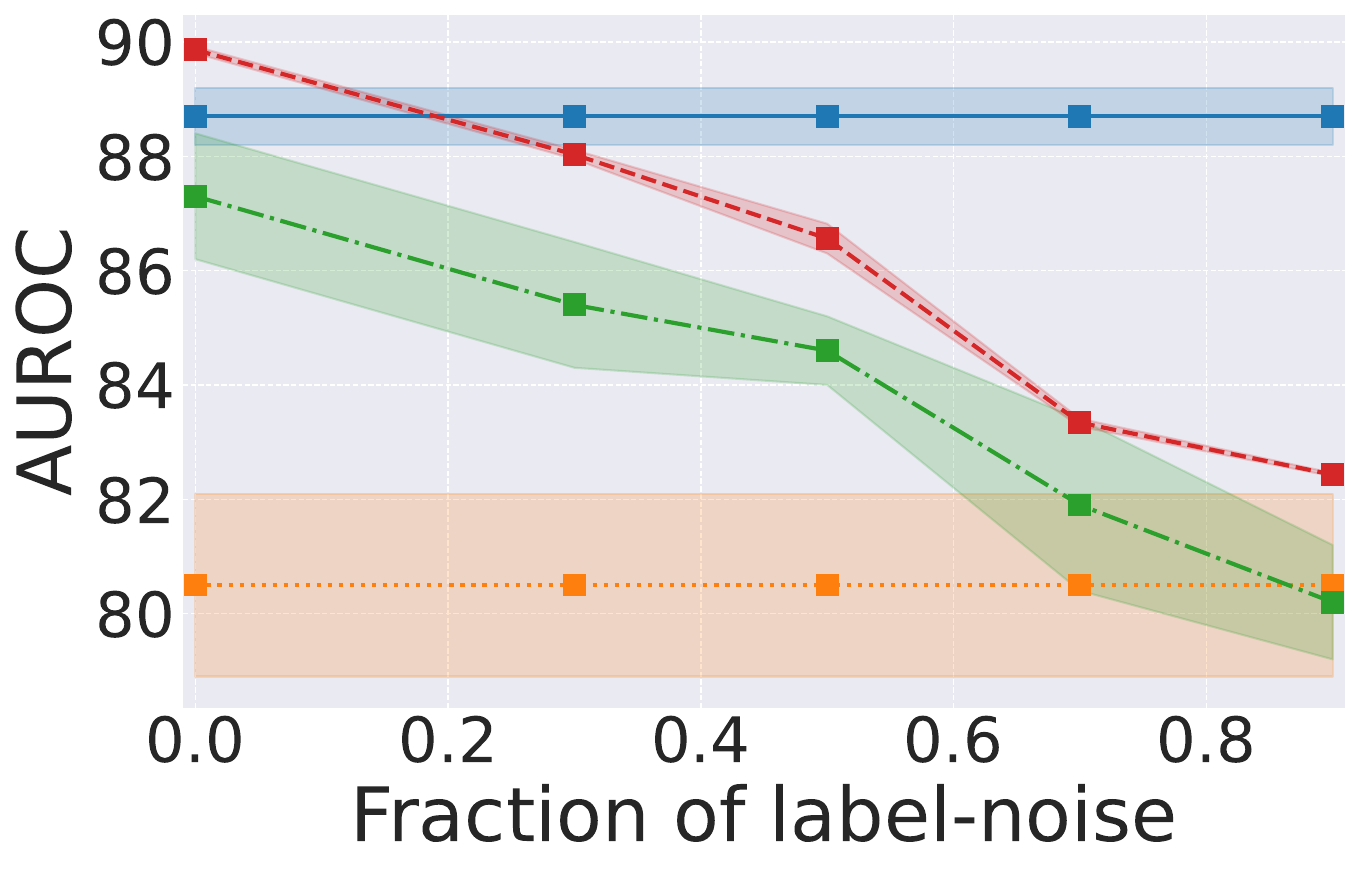}
    \caption{Cora}
    \label{fig:fig2}
  \end{subfigure}
  \hspace{0pt} % No space between figures  
  \begin{subfigure}{0.19\textwidth}
    \centering
    \includegraphics[height=2.3cm, width=\linewidth]{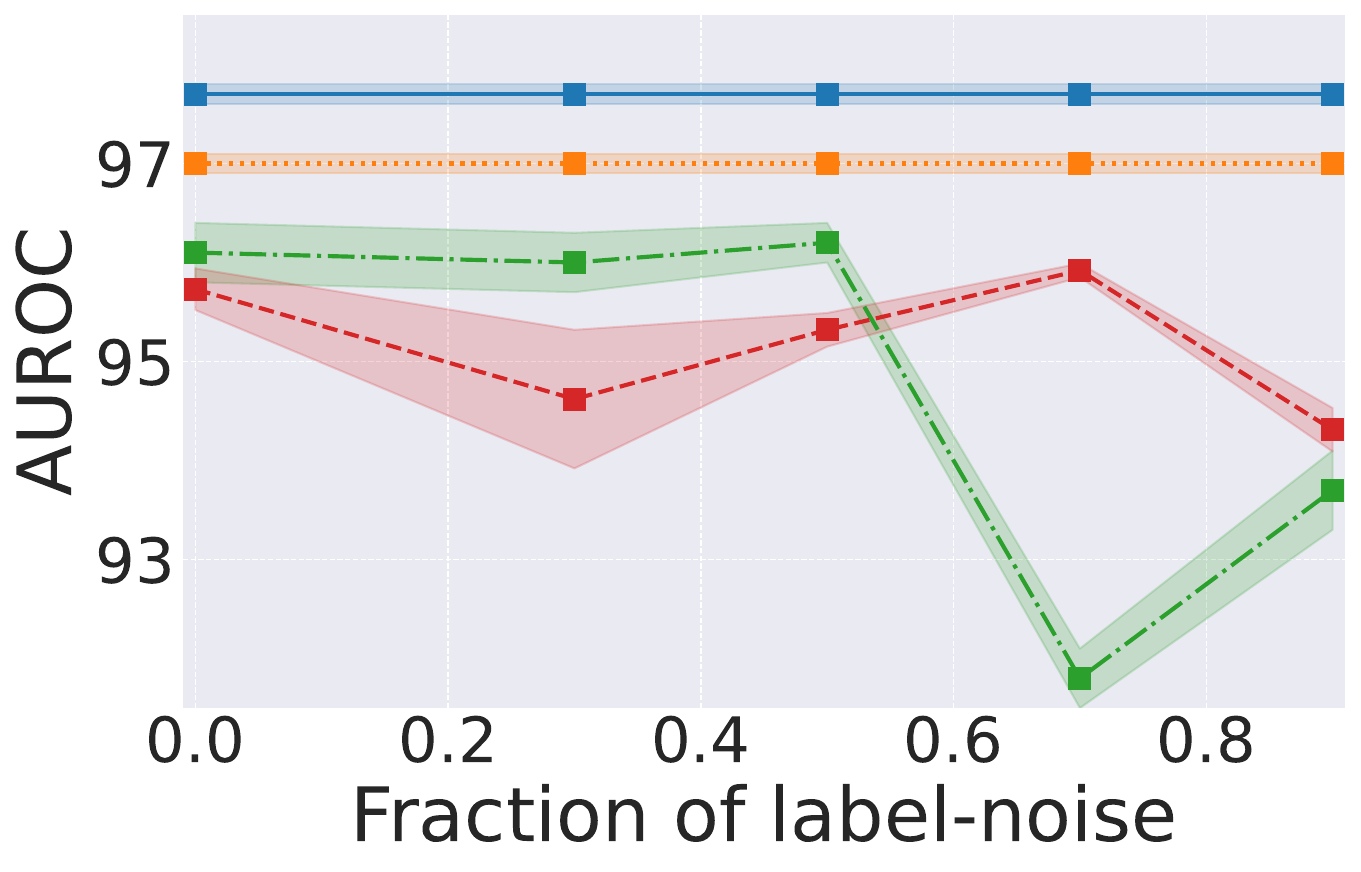}
    \caption{Ogbn-arxiv}
    \label{fig:fig3}
  \end{subfigure}
  \begin{subfigure}{0.19\textwidth}
    \centering
    \includegraphics[height=2.3cm, width=\linewidth]{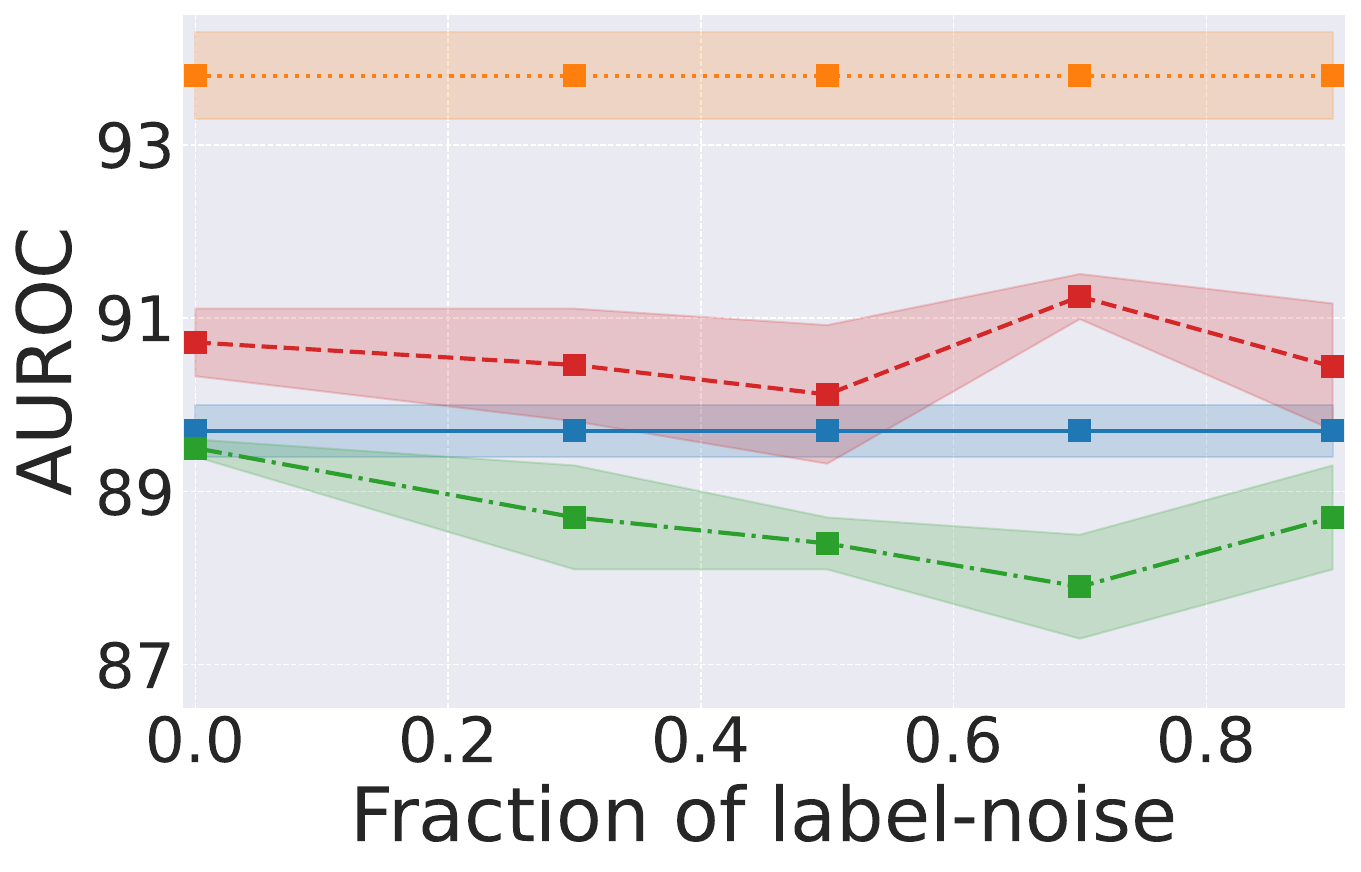}
    \caption{Flickr}
    \label{fig:fig4}
  \end{subfigure}
  \hspace{0pt} % No space between figures in second row
  \begin{subfigure}{0.19\textwidth}
    \centering
    \includegraphics[height=2.3cm, width=\linewidth]{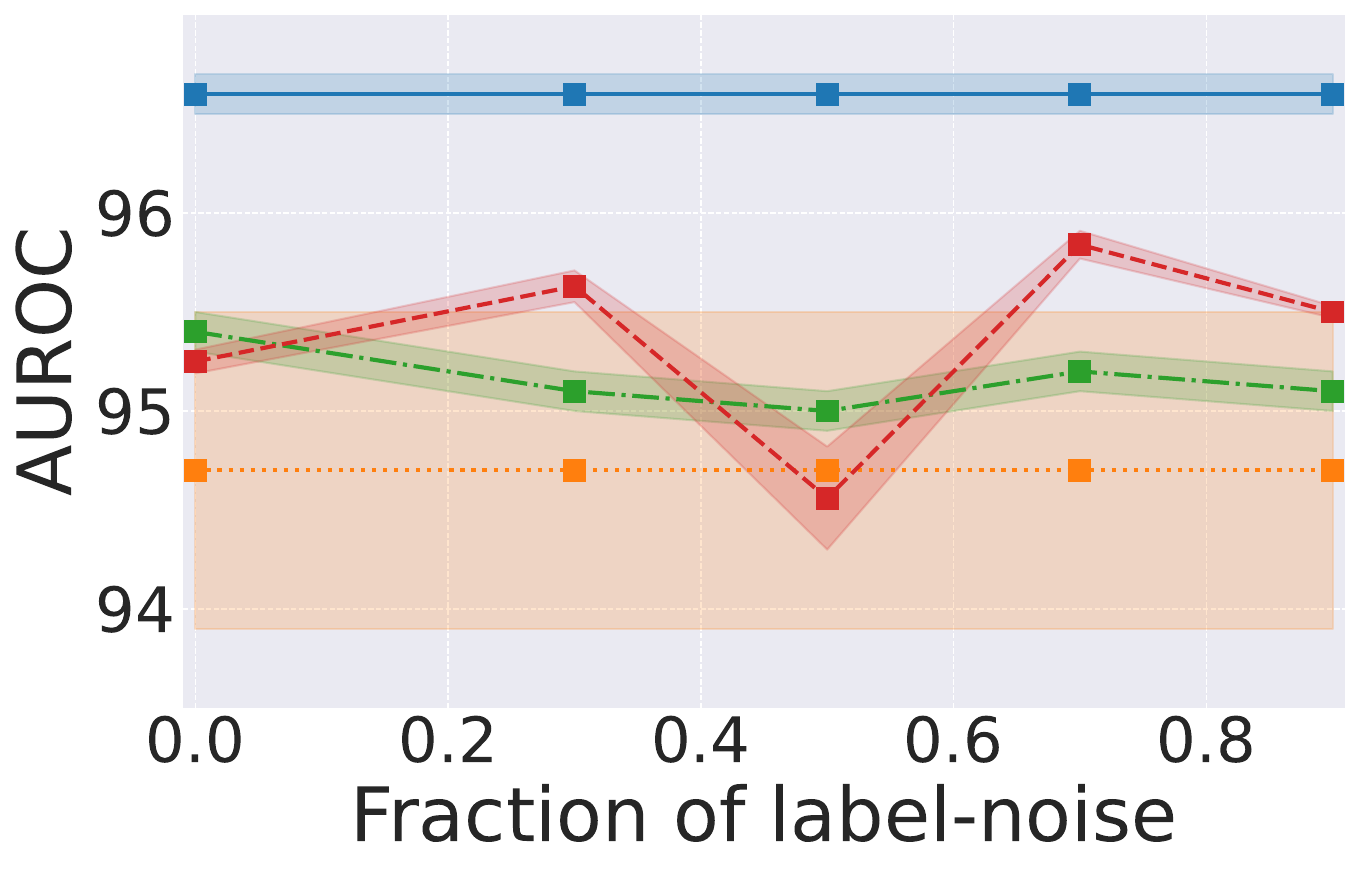}
    \caption{Reddit}
    \label{fig:fig5}
  \end{subfigure}  

  \begin{subfigure}{0.6\textwidth}
    \centering
    \includegraphics[width=\linewidth]{./images/swav_results_noRange/legend.pdf}
    \label{fig_link_legend}
  \end{subfigure}
  \vspace{-1em}
  \caption{\small Single-Source with Label Noise for Link Prediction: PLGC achieves stable performance, demonstrating robustness by utilizing structural and semantic information through self-supervised learning and outperforming supervised methods like GCond and SFGC at larger noise level.
  }
  \label{fig_link_source1}
  % \vspace{-1em}
\end{figure*}

\section{Conclusion} \label{sec_conclusion}
In this work, we introduced \textbf{PLGC}, a self-supervised graph condensation framework that eliminates the reliance on ground-truth labels by constructing pseudo-labels directly from latent node representations. Unlike prior condensation methods that depend on noisy or scarce supervision, PLGC leverages a principled prototype-based pseudo-labeling mechanism to capture latent cluster structure and guide the condensation process in a fully label-free manner.

We provided theoretical guarantees showing that, under mild distributional assumptions, the learned pseudo-labels preserve essential latent-space properties, including cluster separation and concentration of empirical centroids, with high probability. These results formally justify the stability and reliability of pseudo-label–driven condensation and explain why PLGC remains robust even when supervision is severely corrupted or unavailable.

Extensive experiments on five benchmark datasets demonstrate that PLGC achieves performance comparable to state-of-the-art supervised graph condensation methods in clean-label settings, while substantially outperforming them under increasing levels of label noise. Moreover, PLGC consistently excels in few-shot and multi-source scenarios, highlighting its ability to aggregate information across heterogeneous and noisy data sources—an increasingly common setting in real-world graph applications.

Overall, PLGC bridges the gap between self-supervised representation learning and dataset condensation by unifying theoretical rigor with empirical robustness. This work opens several promising directions for future research, including extending pseudo-label–driven condensation to dynamic or heterogeneous graphs, exploring tighter theoretical bounds under weaker assumptions, and applying PLGC to continual learning and privacy-preserving graph learning settings.

\bibliography{main}
\bibliographystyle{tmlr}

\newpage
\appendix

\section{Theoretical Analysis of Pseudo-Labeled Graph Condensation}
\label{sec:appendix-theory}

This appendix provides a complete theoretical analysis of why the pseudo-labels
$\widetilde{Y}=\{\widetilde{y}_1,\dots,\widetilde{y}_K\}$ produced by PLGC behave as assignment-weighted centroids and concentrate around the true latent centers.
Our main theorem shows that:  
(1) each pseudo-label $\widetilde{y}_k$ concentrates around the true latent
center $\mu_k$, and  
(2) the pseudo-labels preserve cluster separation, ensuring that the condensed
graph encodes the same latent geometric structure as the original graph.  
This result explains why PLGC produces stable and robust condensed graphs,
especially under label noise or weak supervision.
We also demonstrate the sample complexity to ensure stronger separation among the learned pseudo-labels.
In the following, we first present a set of supporting lemmas followed by the main theorem for our paper.

\begin{lemma}[\textsf{Existence of a $1/2$-net on unit sphere}]
\label{lem:half-net}
Let $\mathbb{S}^{d-1} := \{u\in\mathbb{R}^d : \|u\|_2=1\}$ denote the unit sphere.
There exists a finite set $\mathcal{N}\subset\mathbb{S}^{d-1}$ such that
for every $u\in\mathbb{S}^{d-1}$, there exists $v\in\mathcal{N}$ satisfying
$\|u-v\|\le \tfrac12$.
Moreover, the cardinality of $\mathcal{N}$ satisfies $|\mathcal{N}|\le 5^d$.
\end{lemma}

\begin{proof}
Construct \(\mathcal{N}\) as any maximal $1/2$--separated subset of \(\mathbb{S}^{d-1}\); that is,
choose points on the sphere greedily so that the distance between any two chosen points is \(>1/2\),
and stop when no further point can be added. By maximality the chosen set is automatically a \(1/2\)-net:
if there existed \(u\in\mathbb{S}^{d-1}\) with \(\|u-v\|>1/2\) for all chosen \(v\), then \(u\) could be added,
contradicting maximality.

It remains to bound the cardinality. For each $v\in\mathcal{N}$ consider the Euclidean ball
\(B(v,r)\subset\mathbb{R}^d\) of radius \(r=\tfrac14\). Because the points in \(\mathcal{N}\) are
$1/2$-separated, the balls \(B(v,\tfrac14)\) are pairwise disjoint. Moreover, every such ball
is contained in the Euclidean ball centered at the origin of radius \(1+\tfrac14=\tfrac54\),
since each \(v\in\mathbb{S}^{d-1}\) has $\|v\|=1$. Hence the union of these disjoint balls is
contained in \(B(0,5/4)\). Taking Lebesgue volumes (and denoting by \(\mathrm{vol}(B(0,r))\)
the volume of a radius-\(r\) ball in \(\mathbb{R}^d\)) we obtain
\[
|\mathcal{N}|\cdot \mathrm{vol}\!\big(B(0,\tfrac14)\big)
\le
\mathrm{vol}\!\big(B(0,\tfrac54)\big).
\]
Using the scaling of Euclidean volume, \(\mathrm{vol}(B(0,r)) = r^d \mathrm{vol}(B(0,1))\),
this yields
\[
|\mathcal{N}| \le \frac{(5/4)^d ~\mathrm{vol}(B(0,1))}{(1/4)^d ~\mathrm{vol}(B(0,1))}
= \Big(\frac{5/4}{1/4}\Big)^d = 5^d.
\]
\end{proof}

\begin{lemma}[\textsf{Directional concentration via Chernoff}]
\label{lem:directional-chernoff-swav}
Let $\{z_i\}_{i=1}^n$ be the set of independent samples corresponding to $k$-th pseudo-label with true latent centers, $\mu_k$.
We denote $q_{ik}\in\{0,1\}$ as the label assignments and define $s_k := \sum_{i=1}^n q_{ik}$.
Assume that for every unit vector $u\in\mathbb{S}^{d-1}$, the scalar random variable $u^\top(z_i-\mu_k)$ is $\sigma$-sub-Gaussian, \textit{i.e.,} 
\[
\mathbb{E}\big[\exp(\lambda\, u^\top (z_i-\mu_k))\big] \le \exp\!\left(\frac{\sigma^2\lambda^2}{2}\right)
\qquad\forall \lambda\in\mathbb{R}.
\]
Then for any $t>0$,
\[
\Pr\!\left(
\left|
u^\top
\left(
\frac{\sum_i q_{ik} z_i}{s_k}
-
\mu_k
\right)
\right|
\ge t
\right)
\le
2\exp\!\left(
-\frac{s_k t^2}{2\sigma^2}
\right).
\]
\end{lemma}

\begin{proof}
We fix $u\in\mathbb{S}^{d-1}$ and define $X_i := u^\top(z_i-\mu_k)$, and $S := \sum_{i=1}^n q_{ik} X_i$.
Then $u^\top\left(\frac{\sum_i q_{ik} z_i}{s_k} - \mu_k \right) = \frac{S}{s_k}$.

Since each $X_i$ is $\sigma$-sub-Gaussian and $q_{ik}\in \{0,1\}$, we have for any
$\lambda>0$,
\[
\mathbb{E}\!\left[\exp\!\left(\lambda q_{ik} X_i\right)\right]
\le
\exp\!\left(\frac{\sigma^2\lambda^2 q_{ik}^2}{2}\right)
\le
\exp\!\left(\frac{\sigma^2\lambda^2 q_{ik}}{2}\right).
\]
Given $z_i$'s and therefore, $X_i$'s are independent, we have
\[
\mathbb{E}[e^{\lambda S}]
=
\prod_{i=1}^n
\mathbb{E}[e^{\lambda q_{ik} X_i}]
\le
\exp\!\left(
\frac{\sigma^2\lambda^2}{2}
\sum_{i=1}^n q_{ik}
\right)
=
\exp\!\left(
\frac{\sigma^2\lambda^2 s_k}{2}
\right).
\]

Next, we apply Markov:
$\Pr(S \ge s_k t) = \Pr(e^{\lambda S} \ge e^{\lambda s_k t}) \le  e^{-\lambda s_k t}~\mathbb{E}[e^{\lambda S}]
\le
\exp\!\left(
-\lambda s_k t + \frac{\sigma^2\lambda^2 s_k}{2}
\right)$.

Since, the above equation is true for any $\lambda \in \mathbb{R}$, we select $\lambda=t/\sigma^2$ to minimize the LHS, yielding
\[
\Pr(S \ge s_k t)
\le
\exp\!\left(-\frac{s_k t^2}{2\sigma^2}\right)
\Longleftrightarrow
\Pr\left(\frac{S}{s_k} \ge t\right)
= 
\Pr\left(u^\top\frac{\sum_i q_{ik} z_i}{s_k} - \mu_k \ge t\right)
\le
\exp\!\left(-\frac{s_k t^2}{2\sigma^2}\right)
.
\]
Finally, applying the same argument to $-S$ (or equivalently to $-u$) yields the two-sided bound as claimed in the lemma:
\[
\Pr\!\left(
\left|
u^\top
\left(
\frac{\sum_i q_{ik} z_i}{s_k}
-
\mu_k
\right)
\right|
\ge t
\right)
\le
2\exp\!\left(
-\frac{s_k t^2}{2\sigma^2}
\right).
\]
\end{proof}

Next, we analyze the stationarity of our pseudo-label learning objective.
Note that the actual loss objective (Eq.~\ref{eq_pseudo_label_learner}) enforces
cross-view consistency by matching pseudo-label assignments of one augmented view
with embeddings of another.
Although this cross-view formulation is crucial for effectively learning the
pseudo-labels, the optimization admits a simpler characterization with respect
to the pseudo-label variables.

Specifically, when considering fixed embeddings and assignment variables, the
objective function reduces to aggregating the similarity between the node
embeddings and their assigned pseudo-labels.
Therefore, without loss of generality, we analyze the stationary points of the
following simplified pseudo-label learning task and show that, under the
first-order optimality conditions, pseudo-labels align with the
assignment-weighted sum of embeddings.

\begin{lemma}[\textsf{Stationarity Condition of pseudo-labels}]
\label{lem:stationarity}
Consider the following simplified pseudo-label learning objective associated
with Eq.~\ref{eq_pseudo_label_learner}:
\[\min_{\{\widetilde y_k\}_{k=1}^K} \; -\sum_{i=1}^n \sum_{k=1}^K q_{ik}\, z_i^\top \widetilde y_k
\quad
\text{subject to } \|\widetilde y_k\|_2 = 1 \;\; \forall k .
\]
Any stationary point $\{\widetilde y_k\}_{k=1}^K$ of this problem satisfies, for each $k$,
$\widetilde y_k
\;\parallel\;
\sum_{i=1}^n q_{ik} z_i$ or
equivalently,
$\widetilde y_k = \frac{\sum_{i=1}^n q_{ik} z_i} {\big\|\sum_{i=1}^n q_{ik} z_i\big\|}$.
\end{lemma}

\begin{proof}
By introduce Lagrange multipliers $\lambda_k \in \mathbb{R}$ for the unit-norm
constraints, we get:
\[
\mathcal{L}
=
-\sum_{i=1}^n \sum_{k=1}^K q_{ik}\, z_i^\top \widetilde y_k
+
\sum_{k=1}^K \lambda_k \big( \|\widetilde y_k\|_2^2 - 1 \big).
\]

At a stationary point, the first-order optimality (KKT) conditions require:
$\nabla_{\widetilde y_k} \mathcal{L} = 0 \quad \forall k$.

Computing the gradient yields: $-\sum_{i=1}^n q_{ik} z_i + 2\lambda_k \widetilde y_k = 0$.
Finally, we can rearrange to obtain $\widetilde y_k = \frac{1}{2\lambda_k} \sum_{i=1}^n q_{ik} z_i$.
Therefore, $\widetilde y_k$ is collinear with $\sum_{i=1}^n q_{ik} z_i$.
Imposing the constraint $\|\widetilde y_k\|_2 = 1$ yields the stated normalized
form.
\end{proof}

\begin{theorem}[\textsf{Pseudo-labels Concentration and Interior Point recovery}]
\label{thm:centroid}
Let $\{z_i\}_{i=1}^n$ be the independent samples corresponding to $K$ latent centers, with pseudo-label assignments denoted by $q_{ik} \in \{0,1\}$.
The samples $\{z_i : q_{ik}=1\}$ are corresponding to pseudo-label $k$.
We denote $\mu_k$ be the true latent centers (or the population centers) and  define $s_k := \sum_{i=1}^n q_{ik}$. 
We assume: 

\begin{enumerate}[label=\textbf{(A\arabic*)}]
  \item \textbf{Sub-Gaussian noise.} for every unit vector $u\in\mathbb{S}^{d-1}$, the scalar random variable $u^\top(z_i-\mu_k)$ is $\sigma$-sub-Gaussian, \textit{i.e.,} 
$\mathbb{E}\big[\exp(\lambda\, u^\top (z_i-\mu_k))\big] \le \exp\!\left(\frac{\sigma^2\lambda^2}{2}\right)
\quad\forall \lambda\in\mathbb{R}$.

  \item \textbf{Separation.} The true latent centers are well-separated i.e., 
  $\displaystyle \Delta := \min_{j\neq k}\|\mu_j-\mu_k\| > 0$.
  
\end{enumerate}

We set the deviation parameter $\epsilon_k \;:=\; 4\sigma\sqrt{\frac{d+\log(2K/\delta)}{s_k}}$.
Given that the pseudo-labels satisfies the stationarity condition as in Lemma \ref{lem:stationarity}, the following statements hold simultaneously for all \(k\) with probability at least \(1-\delta\):

\begin{enumerate}[label=(\roman*)]
  \item \textbf{Pseudo-label Concentration.} 
  Each pseudo-label satisfies $\|\widetilde{y}_k - \mu_k\| \le \epsilon_k$, i.e., pseudo-labels produced by PLGC remains tightly concentrated to the true latent centers.

  \item \textbf{Interior-point recovery.} Any point $z_i$ assigned to $k$-th pseudo-label (i.e., $q_{ik}=1$) such that $\|z_i - \mu_k\| < \frac{\Delta}{2} - \epsilon_k$ remains closer to $\widetilde{y}_k$ than other labels 
  $\widetilde{y}_\ell$ -- indicating that, the interior points are always correctly assigned.

  \item \textbf{Sample Complexity \& Pseudo-Label Separation.}
  A sufficient sample size to ensure a positive interior margin of $\epsilon_{\max}\le\Delta/\beta$ is: $s_k \;\ge\; \frac{16\,\sigma^2\,\beta^2}{\Delta^2}\,\big(d + \log(2K/\delta)\big)$, and leads to a stronger separation of pseudo-labels: $\|\widetilde y_k-\widetilde y_\ell\| \ge (1-\tfrac{2}{\beta})\Delta ~~ \forall k\neq\ell$
  
\end{enumerate}
\end{theorem}

\begin{proof}
We prove each claim in order. The argument uses the standard net + Chernoff
method for sub-Gaussian vectors; the constants above are conservative but explicit.

\paragraph{(i) Concentration of sample centroids.}
Fix a cluster \(k\) and a unit vector \(u\in\mathbb{S}^{d-1}\). By the
sub-Gaussian assumption and directional concentration argument (Lemma \ref{lem:directional-chernoff-swav}),
we get,
\[
\Pr\!\left( \left|
u^\top \left(\frac{\sum_i q_{ik} z_i}{s_k}
-
\mu_k
\right)
\right|
\ge t
\right)
\le 2\exp\!\left(-\frac{s_k t^2}{2\sigma^2} \right)
\qquad \forall t > 0
\]

Denoting \(\mathcal{N}\) be a $1/2$-net of \(\mathbb{S}^{d-1}\), we get $|\mathcal{N}|\le 5^d$ (Lemma \ref{lem:half-net}). 
Applying the tail bound to each $v\in\mathcal{N}$
and using union bound argument, we yield
\[
\Pr\!\Bigg(\exists v\in\mathcal{N}:\ \Bigg|v^\top\Bigg(\frac{\sum_i q_{ik} z_i}{s_k}-\mu_k\Bigg)\Bigg| \ge t\Bigg)
\le \sum_{v \in \mathcal{N}} \Pr\Bigg(\Bigg|v^\top\Bigg(\frac{\sum_i q_{ik} z_i}{s_k}-\mu_k\Big)\Bigg| \ge t\Bigg)
= 2\cdot 5^d \exp\!\Big(-\frac{s_k t^2}{2\sigma^2}\Big).
\]

If the maximum error over the net is $<t$, i.e., $\max_{v\in\mathcal{N}} |v^\top(\bar z_k-\mu_k)| < t$, then for any $u\in\mathbb{S}^{d-1}$, we can
choose $v\in\mathcal{N}$ such that $\|u-v\|\le 1/2$. Now, we using triangle inequality and decomposition $u = v + (u-v)$, we get:
\[
\Bigg|u^\top\Bigg(\frac{\sum_i q_{ik} z_i}{s_k}-\mu_k\Bigg)\Bigg| \le \Bigg|v^\top\Bigg(\frac{\sum_i q_{ik} z_i}{s_k}-\mu_k\Bigg)\Bigg| + \left|(u-v)^\top\Bigg(\frac{\sum_i q_{ik} z_i}{s_k}-\mu_k\Bigg)\right|.
\]

Applying the Cauchy-Schwarz inequality, $\left|\Big(u-v\Big)^\top\Big(\frac{\sum_i q_{ik} z_i}{s_k}-\mu_k\Big)\right| \le \Big\|u-v\Big\| \Big\|\frac{\sum_i q_{ik} z_i}{s_k}-\mu_k\Big\|$, along with the net property $\|u-v\| \le 1/2$ and the assumption that $|v^\top(\bar z_k-\mu_k)| < t$, we obtain:
\[
\Bigg|u^\top\Bigg(\frac{\sum_i q_{ik} z_i}{s_k}-\mu_k\Bigg)\Bigg| < t + \tfrac12\Bigg\|\frac{\sum_i q_{ik} z_i}{s_k}-\mu_k\Bigg\|.
\]
Taking the supremum over all $u \in \mathbb{S}^{d-1}$ yields the norm $\sup_{u} \Big|u^\top\Big(\frac{\sum_i q_{ik} z_i}{s_k}-\mu_k\Big)\Big| = \Big\|\frac{\sum_i q_{ik} z_i}{s_k}-\mu_k\Big\|$, leading to the inequality:
\[
\Big\|\frac{\sum_i q_{ik} z_i}{s_k}-\mu_k\Big\| < t + \Big\|\frac{\sum_i q_{ik} z_i}{s_k}-\mu_k\Big\|.
\]
Rearranging the terms, we get: $\Big\|\frac{\sum_i q_{ik} z_i}{s_k}-\mu_k\Big\| < 2t$. Consequently, the event $\Big\{\Big\|\frac{\sum_i q_{ik} z_i}{s_k}-\mu_k\Big\| \ge 2t \Big\}$ is contained within the event that the net bound fails, leading to the final lifting inequality:
\[
\Pr\!\Bigg(\Big\|\frac{\sum_i q_{ik} z_i}{s_k}-\mu_k\Big\| \ge 2t\Bigg)
\le 2\cdot 5^d \exp\!\Big(-\frac{s_k t^2}{2\sigma^2}\Big).
\]

Choosing
$t = \sigma\sqrt{\frac{2(d\log 5 + \log(2K/\delta))}{s_k}}$ (\textit{i.e.}, \(2t=\epsilon_k\)), we get
\(\Pr(\|\bar z_k-\mu_k\|\ge\epsilon_k)\le \delta/(2K)\). 
A union bound over all pseudo-labels, 
$k=1,\dots,K$ yields claim-(i) with probability at least \(1-\delta/2\). (Absorb the remaining small failure probability to get overall \(1-\delta\).)

\paragraph{(ii) Interior-point recovery.}
Fix \(k\) and a point \(z_i\) with $q_{ik}=1$ and
\(\|z_i-\mu_k\|< \tfrac{\Delta}{2} - \epsilon_k\).
For any other \(\ell\neq k\), we use the triangle inequality:
\[
\|z_i - \widetilde y_k\|
\le \|z_i-\mu_k\| + \|\mu_k - \widetilde y_k\|
< (\tfrac{\Delta}{2}-\epsilon_k) + \epsilon_k = \tfrac{\Delta}{2}.
\]
On the other hand, applying the triangle inequality and (i),
\[
\|z_i - \widetilde y_\ell\|
\ge \|\mu_k - \mu_\ell\| - \|z_i-\mu_k\| - \|\widetilde y_\ell - \mu_\ell\|
> \Delta - (\tfrac{\Delta}{2}-\epsilon_k) - \epsilon_k
= \tfrac{\Delta}{2}.
\]

Therefore, we always have \(\|z_i-\widetilde y_k\| < \|z_i-\widetilde y_\ell\|\) for all \(\ell\neq k\).
Hence, \(z_i\) is correctly assigned to \(\widetilde y_k\).

\paragraph{(iii) Sample Complexity \& Pseudo-Label Separation.} 
The sufficient sample size to ensure  \(\epsilon_{\max}\le \Delta/\beta\) is obtained by rearranging the expression of $\epsilon_k \;:=\; 4\sigma\sqrt{\frac{d+\log(2K/\delta)}{s_k}}$ \textit{i.e.,} 
\[
s_k \;\ge\; \frac{16\,\sigma^2\,\beta^2}{\Delta^2}\,\big(d + \log(2K/\delta)\big)
\qquad\text{for all }k.
\]

Finally, \(\epsilon_{\max}\le \Delta/\beta\) naturally leads to a stronger separation bound of pseudo-labels by using triangle inequality as follows:
\[
\|\widetilde y_l-\widetilde y_k\|
\ge \|\mu_k - \mu_\ell\| -  \|\widetilde y_k - \mu_k\| - \|\widetilde y_\ell - \mu_\ell\|
\ge \Delta - \epsilon_k - \epsilon_\ell 
\ge \Delta - 2\epsilon_{\max} = \Big(1-\frac{2}{\beta}\Big)\Delta
\]

\end{proof}

\end{document}